\documentclass[acmsmall]{acmart}

\AtBeginDocument{%
  \providecommand\BibTeX{{%
    \normalfont B\kern-0.5em{\scshape i\kern-0.25em b}\kern-0.8em\TeX}}}

\acmConference[]{}
\acmBooktitle{}
\acmPrice{}
\acmISBN{}

\usepackage{algorithm}
\usepackage{algorithmic}
\usepackage{subfigure}
\usepackage{multicol}
\usepackage{multirow}
\begin{document}

\title{An Efficient Learning Framework For Federated XGBoost Using Secret Sharing And Distributed Optimization}

\author{Lunchen Xie}
\affiliation{%
 \institution{School of Software Engineering, Tongji University}
 \streetaddress{Cao'an Highway 4800}
 \city{Shanghai}
 \state{Shanghai}
 \country{China}}
\email{lcxie@tongji.edu.cn}

\author{Jiaqi Liu}
\affiliation{%
 \institution{School of Software Engineering, Tongji University}
 \streetaddress{Cao'an Highway 4800}
 \city{Shanghai}
 \state{Shanghai}
 \country{China}}
\email{liujustdoit@outlook.com}

\author{Songtao Lu}
\affiliation{%
 \institution{IBM Research AI, IBM Thomas J. Watson Research Center}
 \streetaddress{1101 Kitchawan Road, 10598}
 \city{Yorktown Heights}
 \state{New York}
 \country{United States}}
\email{songtao@ibm.com}

\author{Tsung-Hui Chang}
\affiliation{%
 \institution{School of Science and Engineering, The Chinese University of Hong Kong, Shenzhen}
 \streetaddress{Longxiang Avenue 2001}
 \city{Shenzhen}
 \state{Guangdong}
 \country{China}}
\email{tsunghui.chang@ieee.org}

\author{Qingjiang Shi}
\affiliation{%
 \institution{School of Software Engineering, Tongji University}
 \streetaddress{Cao'an Highway 4800}
 \city{Shanghai}
 \state{Shanghai}
 \country{China}}
\email{shiqj@tongji.edu.cn}

\renewcommand{\shortauthors}{Xie, et al.}

\begin{abstract}
  XGBoost is one of the most widely used machine learning models in the industry due to its superior learning accuracy and efficiency. Targeting at data isolation issues in the big data problems, it is crucial to deploy a secure and efficient federated XGBoost (FedXGB) model. Existing FedXGB models either have data leakage issues or are only applicable to the two-party setting with heavy communication and computation overheads. In this paper, a lossless multi-party federated XGB learning framework is proposed with a security guarantee, which reshapes the XGBoost's split criterion calculation process under a secret sharing setting and solves the leaf weight calculation problem by leveraging distributed optimization. Remarkably, a thorough analysis of model security is provided as well, and multiple numerical results showcase the superiority of the proposed FedXGB compared with the state-of-the-art models on benchmark datasets.
  %. Finally, we verify the efficacy of our model on benchmark datasets.
\end{abstract}





\maketitle

\section{Introduction}\label{sec1}
It is well known that XGBoost \cite{chen2016xgboost} is one of the most popular machine learning algorithms, which has been been widely used as a standard data modeling tool in world-famous Kaggle competitions as well as various real-world applications, e.g., product recommendation \cite{he2014practical}, fraud detection \cite{oentaryo2014detecting}, and online advertisement \cite{ling2017model}, etc. This method is generally implemented with centralized data storage. As the number of data sources increases, there would be a rising demand for collaboration of data from different organizations to build a more powerful XGBoost model. However, due to commercial competition and data privacy issues, raw data sharing among different organizations is strictly forbidden, leading to the so-called data isolation problem \cite{yang2019federated} and making centralized XGBoost infeasible.

To address the data isolation problem, federated learning was firstly proposed by Google \cite{konevcny2016federated} and then rapidly drew a great deal of attentions \cite{hardy2017private,liu2019boosting,cheng2019secureboost,gu2020federated,yang2019federated}. So far, there are mainly two federated learning schemes: \emph{horizontal federated learning} and \emph{vertical federated learning}. To be more specific, \emph{Horizontal federated learning} was proposed for the scenario where different data holders share the same feature space but different samples, while \emph{vertical federated learning} was proposed to handle the case when data holders share the same entity set but different feature space. This paper mainly focuses on \emph{vertical federated learning} for XGBoost. Our goal is to build a federated XGBoost model (FedXGB) jointly from vertically partitioned data held by different participants while ensuring data privacy. 

In the literature, many federated learning algorithms have been proposed for gradient tree boosting \cite{cheng2019secureboost,fang2020hybrid,Feng2020SecureGBM,liu2020federated,leung2020towards}. Among them, \cite{cheng2019secureboost,fang2020hybrid} are the two most closely related works to FedXGB on vertically partitioned data. The lossless privacy-preserving algorithm proposed by Cheng et al. \cite{cheng2019secureboost} is to train a high-quality tree boosting model using homomorphic encryption (HE) \cite{gentry2009fully}, but it may leak intermediate information like instance index and order of loss reductions in the training process. In addition, the encryption overhead might not be tolerable in the massive data scenario. To alleviate intermediate information leakage, Fang et al. \cite{fang2020hybrid} presented a framework based on secret sharing (SS) \cite{blakley1979safeguarding} to build secure XGBoost with vertically partitioned data. Since under SS all participants build local models only based on data slices (rather than raw data) sent from others, the SS-based FedXGB has better performance in terms of privacy and security. However it is only applicable to the case of two participants, and thus has limited applications. Moreover, the lossless division operation involved in SS-based XGBoost requires high computational complexity. In short, existing vertical FedXGB models either have the data leakage issue or are only applicable to the two-party setting with heavy communication and computation overheads, even though SS might have some privacy and security guarantees. These facts motivate us to study a secure and efficient vertical \emph{multi-party} SS-based FedXGB framework.

However, because the training process of the vanilla XGBoost model requires non-linear calculation, e.g., \emph{argmax} and \emph{division}, there exist two major following challenges in the adaptation of XGBoost to an efficient multi-party vertical federated learning model under the SS setting. First, the \emph{argmax} operation is indispensable to find the best split in XGBoost. In SS, all participants cannot have direct access to actual loss reductions of split candidates for security concerns. For two split candidates, the existing implementation of SS-\emph{argmax} \cite{fang2020hybrid} is based on judging the sign of the difference between their loss reductions. This process is conducted by comparing two shares of the difference bit by bit through multiplexers in the \emph{two-party} federated learning scenario, unfortunately, it is not feasible for the \emph{multi-party} scenario. Second, computing split criterion and leaf weight both need \emph{division} operations, whereas the SS approach cannot deal with \emph{division} operations directly by simply utilizing existing additive, subtraction, and multiplication primitives defined in SS. Complicated composition and iteration of these primitives can approximate \emph{division}, but requires high computational complexity. These two challenges hinder the development of efficient \emph{multi-party} vertical FedXGB.

To address the above challenges, we propose a lossless \emph{multi-party} vertical FedXGB learning framework for both classification and regression tasks, named as MP-FedXGB\footnote{https://github.com/HikariX/MP-FedXGB}. Specifically, to enable secure \emph{multi-party} collaborative modeling, we redesign the \emph{argmax} in the split criterion calculation procedure. Note that the expression of the difference between the loss reductions of every two split candidates contains complicated fractions. By reducing the fractions to a common denominator we reshape them to only one single fraction. Then, through judging on signs of numerator and denominator respectively, we can obtain the result of the \emph{argmax} operation without executing \emph{division} and using multiplexers, thus this way is suitable for the \emph{multi-party} scenario. Besides, we reshape the closed-form leaf weight calculation as minimization of a convex quadratic optimization problem, and develop a gradient-based distributed algorithm to solve the quadratic problem. It is worthy noting that the proposed reshaping above removes the requirement of division operation in the tree-building process, which results in boosting the efficiency of our framework greatly. In addition, we point out the potential security risks and put forward the corresponding solution. Finally, we demonstrate the efficiency of our FedXGB model by conducting multiple experiments on benchmark datasets. Our contributions are highlighted as follows:
\begin{itemize}
\item To the best of our knowledge, MP-FedXGB is the first \emph{multi-party} federated XGB learning framework on vertically partitioned data under the SS setting with high efficiency and scalability.
\item We propose a simple but very effective computation reshaping method for the calculation of split criterion and leaf weight, greatly boosting the training efficiency while preserving data security and privacy.
\item In order to solve the potential instance space leakage problem completely, we propose an extra security mechanism called \emph{First-Layer-Mask} for our framework, further enhancing the security of our framework.
\end{itemize}

\section{Related work}\label{sec2}
Recently, a number of federated learning methods have been proposed to tackle vertically partitioned data. Among them, some are about vertical federated XGBoost. In this section, we give a brief review of vertical federated learning, especially the vertical federated XGBoost, which is the focus of this paper.

\subsection{Vertical Federated Learning}
Vertically partitioned data \cite{skillicorn2008distributed} widely exists in modern data mining and machine learning applications, where data are provided by multiple providers, of which each maintains the records of different feature sets with common entities. The goal of vertical federated learning is to build the model jointly by adopting all the data while ensuring privacy for each data owner. In the literature, there are many privacy-preserving federated learning algorithms for vertically partitioned data in various applications, for example, federated linear regression \cite{karr2009privacy,gascon2017privacy}, federated logistic regression \cite{hardy2017private}, federated random forest\cite{liu2020federated}, federated XGBoost \cite{cheng2019secureboost,fang2020hybrid}, and federated support vector machine \cite{yu2006privacy,gu2020federated}.

Most of the above federated learning algorithms require only the matrix addition and multiplication primitives under the SS mechanism. However, FedXGB needs additional non-linear computation primitives, e.g., division and \emph{argmax}. Thus, the secure design proposed in \cite{hardy2017private} cannot be straightforwardly extended to FedXGB. To tackle the nonlinear computation issue, the work \cite{hardy2017private} proposed the Taylor expansion based method so that the nonlinear operation can be approximately implemented using the addition and multiplication primitives. However, this approximation could  result in accumulated errors, leading a decrease of the prediction accuracy eventually. By eliminating the nonlinear operations, our proposed vertically FedXGB algorithm is lossless in theory.

\subsection{Vertical Federated XGBoost}
Among the existing federated learning algorithms, \cite{cheng2019secureboost} and \cite{fang2020hybrid} are the two most related works, where HE-based FedXGB and SS-based FedXGB were proposed respectively based on different security schemes.

\textbf{HE-based FedXGB} 
Homomorphic encryption (HE) is a form of encryption allowing one to perform calculations on encrypted data without decrypting it. HE uses a public key to encrypt data and allows only the individual with the matching private key to access its unencrypted data. Thus HE is widely used in vertically federated XGBoost to preserve data privacy. Cheng et al. proposed a lossless HE-based privacy-preserving tree-boosting system called SecureBoost \cite{cheng2019secureboost} for vertically federated XGBoost.  However, this algorithm only promises no raw-data sharing, but would suffer intermediate information leakage. We can easily find that the second-order derivatives of data instances are plaintext for the participant who holds the label in SecureBoost. For regression problems and common mean squared error (MSE) loss function, the second-order derivative is 1 for every instance. So the sum of second-order derivatives of instances in one bucket indicates the count of instances in this bucket. If we use an equal-distance bucket for every feature to split the data, a bigger sum means more data in the bucket. If the split is done by another participant, the participant with the label can know the density in each bucket of this feature from others, thus the density distribution is leaked. In addition, the HE procedure is extremely slow and requires a large amount of memory, leading to low efficiency of possessing.

\textbf{SS-based FedXGB} 
SS was introduced independently by Shamir \cite{shamir1979share} and Blakley \cite{blakley1979safeguarding} in 1979. The idea is that one participant has private data which distributes among other participants in a way that none of them alone can recover the data. SS is also widely used in federated learning. To alleviate data leakage issue of \cite{cheng2019secureboost}, SS is introduced for FedXGB \cite{fang2020hybrid}. It contains additive, subtraction, and multiplication primitives, which support computation distributed in each participant using local shares of different data, and can prevent participants from guessing the original values. However, XGBoost needs non-linear operations, e.g., \emph{argmax} and \emph{division}, which are not defined directly in SS. For example, to find the best split when constructing a tree, operator \emph{argmax} is executed to select the split candidate with the biggest loss reduction. Vanilla methods compute all loss reductions, and compare the actual value between every two candidates. Since SS cannot access the original values, judging the sign of difference of two-loss reductions can also reach the goal. Current implementation \cite{fang2020hybrid} computes shares of the difference on two participants, and uses multiplexers to fulfill sign determination, which is inapplicable for multi-party. Besides, the difference of loss reductions and leaf weights both contain \emph{division} operation, while SS has no corresponding computation primitives. \cite{fang2020hybrid} takes Goldschmidt's method \cite{goldschmidt1964applications} to approximate \emph{division} by addition, subtraction, and multiplication, defined as primitives in SS. But each \emph{division} requires many iterations of the above primitives' combination to converge, and calculating differences contains a great deal of \emph{division}s as well, which further increases computational complexity. What's worse is that this approximation demands a carefully chosen initial value as the prerequisite of convergence, which is hardly feasible in practice.

In summary, adopting SS in FedXGB can avoid the disclosure of intermediate data comparing to HE-based XGB. However, the existing SS-based FedXGB is only applicable to the two-party setting with heavy communication and computation overheads. Moreover, to the best of our knowledge, there exists no secure multi-party FedXGB yet so far. This paper aims to provide the first multi-party federated XGB learning framework on vertically partitioned data under an SS setting with high computational efficiency and scalability.

\section{Problem Definition and Preliminaries}\label{sec3}
In this section, we first state our problem. Then we present a brief review of XGBoost with an outline of challenges in its adaptation to the federated setting. At last, we will describe computation primitives used to ensure security in our proposed vertically federated XGBoost.

\subsection{Problem Statement}\label{sec3.1}
We consider a set of $M$ distributed data holders $\{P_m\}_{m=1}^M$ (also referred to as participants) who want to train a XGBoost model by consolidating their respective dataset matrices $\{\boldsymbol{X}_m\in\mathbb{R}^{N\times J_m}\}_{m=1}^M$ with label vectors $\boldsymbol{y}_{N\times 1}$, where N denotes the number of instances, $m$ is the index of holders, and $J_m$ represents the size of feature set of $P_m$. For each participant $P_m$, it has feature set $\mathcal{F}_m$, which contains $J_m$ features. We allow some overlap among different feature sets $\mathcal{F}_m$, and define $J\triangleq|\cup_{m=1}^M\mathcal{F}_m|$. Moreover, we assume that the participants have aligned their data instances.

Our goal is to fully utilize data from multiple participants to build a secure vertical federated XGBoost. To achieve this goal, we define the following roles relevant to our framework.

\begin{definition}\label{definition1}
\textbf{Active participant:} Active participant is the one who wants to build a preferable model and seeks for collaboration. We denote this participant as $P_1$. It holds both data matrix $\boldsymbol{X}_1$ and label vector $\boldsymbol{y}$.  For supervised learning task based on XGBoost, access to label is necessary, but here we suppose that only $P_1$ can visit it.
\end{definition}
\begin{definition}\label{definition2}
\textbf{Auxiliary participant:} Auxiliary participants are the ones who have different features and are invited for federated modeling. We denote them as $P_m, m=2,...,M$. Each of them holds $\boldsymbol{X}_m$ respectively but cannot access to the label vector and any intermediate value from $P_1$.
\end{definition}
\begin{definition}\label{definition3}
\textbf{Coordinator:} In our framework, a third party is introduced for coordination such as generating Beaver's triples \cite{1992Efficient} and allocating permuted feature indices. We denote it as $C$. Since the coordinator has limited access to calculated results in the tree-building process and no access to raw data, it does not harm the system security.
\end{definition}

Some important notations are listed in Table~\ref{tab:table1}. The problem of federated XGBoost modeling can be stated as follows:\\
\textbf{Given:} Data matrices $\{\boldsymbol{X}_m\in \mathbb{R}^{N\times J_m}\}_{m=1}^M$ distributed on $M$ different participants $P_m$ and $\boldsymbol{y}_{N\times 1}$ stored on active participant $P_1$.\\
\textbf{Learn:} An ensemble model that stored partially on different participants containing $T$ trees. Their combination forms a complete XGBoost model. All the partial models have the same structure, while split information is only available to the data holder who owns the feature.\\
\textbf{Security premise:} Our proposed training method should be secure in the semi-honest adversarial setting. That is, the adversarial participants try to infer sensitive information by using all data handled by themselves. This setting is commonly considered by various ML algorithms \cite{2018ABY,2017SecureML}. Specifically, it is required that active participant doesn't collude with others.

\begin{table}[htbp]
    \centering
    \caption{Description of Notation}
    \begin{tabular}{ccl}
        \toprule
        Symbol & Description\\
        \midrule
        $M$ & number of participants\\
        $C$ & coordinator\\
        $P_1$ & active participant\\
        $P_m$ & auxiliary participants, $m=2,3,...,M$\\
        $N$ & size of instance set\\
        $J_m$ & size of feature set on participant $m, m=1,2,...,M$\\
        $g_i$ & first-order derivative of the $i$-th data instance\\
        $h_i$ & second-order derivative of the $i$-th data instance\\
        $\boldsymbol{X}_m\in \mathbb{R}^{N\times J_m}$ & data matrix of the $m$-th participant, $m=1,2,...,M$\\
        $\boldsymbol{y}\in \mathbb{R}^{N\times1}$ & label vector\\
        $T$ & number of trees in MP-FedXGB model\\
        \bottomrule
    \end{tabular}
    \label{tab:table1}
\end{table}

\subsection{A Brief Review of XGBoost \cite{chen2016xgboost}}\label{sec3.2}
Given a dataset $\boldsymbol{X}\in \mathbb{R}^{N\times J}$ with $N$ instances and $J$ features, XGBoost predicts the $i$-th instance $\boldsymbol{x}_i\in \mathbb{R}^{1\times J}$ by using $T$ regression functions, i.e.,
\begin{equation}\label{equation1}
    \hat{y}_i=\sum_{t=1}^T f_t(\boldsymbol{x}_i).
\end{equation}

The tree ensemble model in Eq.~\eqref{equation1} includes functions as
parameters and cannot be optimized using traditional optimization methods in Euclidean space. Instead, XGBoost is trained in an additive manner by calculating $\hat{y}_i^{(t)}=\hat{y}_i^{(t-1)}+f_t(\boldsymbol{x}_i)$. A second-order Taylor expansion is used to approximate the loss function $l(y_i, \hat{y}_i^{(t-1)}+f_t(\boldsymbol{x}_i))$ at the $t$-th iteration as follows:
\begin{equation}\label{equation2}
    \mathcal{L}^{(t)}\simeq\sum_{i=1}^n\left[l\left(y_i, \hat{y}_i^{(t-1)}\right)+g_if_t(\boldsymbol{x}_i)+\frac{1}{2}h_if_t^2(\boldsymbol{x}_i)\right]+\Omega(f_t).
\end{equation}
Here, $l(\cdot)$ is a differentiable convex loss function. For example, MSE loss is used for regression tasks and log-loss is for classification tasks. $\Omega(f_t)\triangleq\gamma U+\frac{1}{2}\lambda\vert\vert w\vert\vert^2$ is the regularization function, where $U$ is the number of leaves in the tree, $\gamma$ and $\lambda$ are parameters used to suppress tree number and weights respectively. $g_i=\partial\, l\left(y_i,\hat{y}_i^{(t-1)}\right)$ and $h_i=\partial^2\,l\left(y_i,\hat{y}_i^{(t-1)}\right)$ are the first- and second-order gradient statistics of the loss function at $\hat{y}_i^{(t-1)}$.

Normally enumeration of all the possible tree structures is impossible. Instead the model starts from a single leaf node containing all instances. Then the node recursively splits the current instance set $I$ to the left and right subset, denoted as $I_L$ and $I_R$ respectively. The loss reduction after the split is given by
\begin{equation}\label{equation3}
    \mathcal{L}_{split}\triangleq\frac{1}{2}\left[\frac{(\sum_{i\in I_L}g_i)^2}{\sum_{i\in I_L}h_i+\lambda}+\frac{(\sum_{i\in I_R}g_i)^2}{\sum_{i\in I_R}h_i+\lambda}-\frac{(\sum_{i\in I}g_i)^2}{\sum_{i\in I}h_i+\lambda}\right]-\gamma.
\end{equation}

The best split can be selected out by comparing $\mathcal{L}_{split}$s of all candidate splits, i.e., the split with the biggest $\mathcal{L}_{split}$. 
When the stop condition (no positive loss reductions or max depth reached) is met, each leaf $u$ can calculate its weight $w$ according to the following equation
\begin{equation}\label{equation4}
    w=-\frac{\sum_{i\in I_u}g_i}{\sum_{i\in I_u}h_i+\lambda}.
\end{equation}
Again, SS can't conduct \emph{division} directly, so this calculation should be approximated apparently, causing high time consumption.

From above, it is observed that, both Eq.~\eqref{equation3} and Eq.~\eqref{equation4} involve many \emph{division} operations. This poses challenges for federated XGBoost in SS because there is no existing SS primitives defined for \emph{division}. In \cite{fang2020hybrid}, all the fractions in Eq.~\eqref{equation3} and Eq.~\eqref{equation4} need to be approximated by iterative algorithms with the existing SS primitives, and thus the \emph{division} operation is time-consuming under SS. Moreover, the existing split finding method in \cite{fang2020hybrid} requires bit-by-bit comparison after computation of Eq.~\eqref{equation3}, which is not feasible for multi-participant federated learning XGBoost. This paper proposes some efficient methods to address the above challenges.

\subsection{Secret Sharing Computation Primitives}\label{sec3.3}
SS has recently been introduced in \cite{2018ABY} for \emph{multi-party} privacy-preserving data mining. In the following, we take advantage of SS to build our secure XGBoost. Under SS, computations are conducted on data shares. Intermediate values produced during the computation are insensitive, but the combination (through additive operation) of corresponding shares returns the computation result on original data. Thus some function values can be computed independently without exchanging sensitive information. We use $\langle\cdot\rangle$ to denote shared data slices. For example, $\langle x\rangle^1$, $\langle x\rangle^2$, and $\langle x\rangle^3$ means three shares of $x$ distributed on participants $1$, $2$, and $3$. Detailed SS computation primitives are provided as follows.

\textbf{SHR}: Whenever $P_m$ wants to share its data $x$ to others, it generates $M-1$ random numbers and denote them as $\langle x\rangle^{m'}$, $m'=1,...,M$, $m'\neq m$. Then $P_m$ sends all the shares to corresponding participants $P_{m'}$. Specially, the share $\langle x\rangle^m$ is computed by $x-\sum_{m'\neq m}^M\langle x\rangle^{m'}$ and stored locally as $P_m$'s own share. Thus $x=\sum_{m=1}^M\langle x\rangle^m$. We denote the above operation as $\langle x\rangle=\textbf{SHR}(x)$. If it is executed on $P_m$, it means $P_m$ distributes data to $P_{m'}$s, while $P_{m'}$s execute \textbf{SHR} only to receive the shares from $P_m$. For ease of exposition, we don't specify the executor apparently for the \textbf{SHR} primitive and also other primitives.

\textbf{ADD}: For two data $x$ and $y$, shared values $\langle x\rangle^m$ and $\langle y\rangle^m$ are stored in $P_m$. To compute $z=x+y$, every $P_m$ first computes $\langle z\rangle^m=\langle x\rangle^m+\langle y\rangle^m$ locally, and then combines $\langle z\rangle^m$ to yield $z=x+y$ where the second equality follows clearly from the definition of the \textbf{SHR} primitive. We denote \textbf{ADD} as $\langle z\rangle=\langle x\rangle+\langle y\rangle$, the same as ordinary addition operation.

\textbf{SUB}: This primitive can be expressed similarly as \textbf{ADD} by changing the operand. We denote \textbf{SUB} as $\langle z\rangle=\langle x\rangle-\langle y\rangle$, same as ordinary subtraction operation. It holds that $z=\sum_{m=1}^M\langle z\rangle^m=x-y$.

\textbf{MUL}: To calculate $z=x*y$, we utilize Arithmetic multiplication triple \cite{1992Efficient}. First, coordinator $C$ will generate three numbers $a$, $b$, and $c=a*b$. Then it generates shares $\{\langle a\rangle^m\}_{m=1}^M=\textbf{SHR}(a)$ and sends to $P_m, m=1,...,M$, the same for $b$ and $c$. $P_m$ computes $\langle e\rangle^m=\langle x\rangle^m-\langle a\rangle^m$ and $\langle f\rangle^m=\langle y\rangle^m-\langle b\rangle^m$ and sends them back to $P_1$. After that, $P_1$ recovers the value $e=\sum_{m=1}^M\langle e\rangle^m$ and $f=\sum_{m=1}^M\langle f\rangle^m$, broadcasts them to all other $P_m$s. Finally, $P_1$ computes $\langle z\rangle^1=e*f+f*\langle a\rangle^1+e*\langle b\rangle^1+\langle c\rangle^1$, while others compute $\langle z\rangle^m=f*\langle a\rangle^m+e*\langle b\rangle^m+\langle c\rangle^m$. It's easy to see that $\sum_{m=1}^M\langle z\rangle^m=x*y=z$. However, without knowing $a$ and $b$, $P_m$ won't get information of $x$ or $y$ from $e$ and $f$. We denote the above multiplication as $\langle z\rangle=\langle x\rangle\otimes\langle y\rangle$. 

The above primitives can be used in element-wise matrix computation. Moreover, it can be concluded that operations using secret sharing computation primitives have the same operational rules as arithmetic operations, which is stated in Theorem 3.4. Note that we generally suppress the superscript of $\langle\cdot\rangle$ to simplify notation.

\begin{theorem}\label{theorem1}
Operations using secret sharing computation primitives have the same operational rules as arithmetic operations.
\end{theorem}
\begin{proof} 
Basic operational rules include commutative law, associative law, and distributive law. Since \textbf{ADD} and \textbf{SUB} are exactly the same as common addition and subtraction, commutability and associativity are naturally satisfied. For \textbf{MUL}, if the original data are $x_1$ and $x_2$, for each participant $P_m$, $x_1*x_2$ is expressed as $\sum_{m=1}^M\langle x_1\rangle^m\otimes\langle x_2\rangle^m=\sum_{m=1}^M\langle x_1*x_2\rangle^m=\sum_{m=1}^M\langle x_2*x_1\rangle^m=\sum_{m=1}^M\langle x_2\rangle^m\otimes\langle x_1\rangle^m$, both compute $x_1*x_2$, thus the commutative law is satisfied. As for associativity, suppose we have $x_1$, $x_2$, and $x_3$, what we expect to get is $x_1*x_2*x_3$. The associativity law follows because we have:
\begin{equation}\label{equation5}
\begin{aligned}
    &\sum_{m=1}^M(\langle x_1\rangle^m\otimes\langle x_2\rangle^m)\otimes\langle x_3\rangle^m
    \\
    =&\sum_{m=1}^M\langle x_1*x_2\rangle^m\otimes\langle x_3\rangle^m\\
    =&\sum_{m=1}^M\langle x_1*x_2*x_3\rangle^m\\
    =&\sum_{m=1}^M\langle x_1\rangle^m\otimes\langle x_2*x_3\rangle^m\\
    =&\sum_{m=1}^M\langle x_1\rangle^m\otimes(\langle x_2\rangle^m\otimes\langle x_3\rangle^m).
\end{aligned}
\end{equation}
Here the second equation means shares of multiplication between $x_1*x_2$ and $x_3$ can restore the same value as $x_1*x_2*x_3$, if we treat $x_1*x_2$ as a whole term. So associative law on \textbf{MUL} primitive is the same as ordinary multiplication. What's more, distributivity law follows because we have $\sum_{m=1}^M\langle x_1*(x_2\pm x_3)\rangle^m=\sum_{m=1}^M\langle x_1*x_2\pm x_1*x_3\rangle^m=\sum_{m=1}^M\left(\langle x_1*x_2\rangle^m\pm\langle x_1*x_3\rangle^m\right)$.
\end{proof}

The above results imply immediately the following theorem.

\begin{theorem}\label{theorem2}
Operations using SS computation primitives are completely lossless.
\end{theorem}
\begin{proof}
From theorem~\ref{theorem1} we know all the primitives have the same operational rules as the ordinary arithmetic computations, so the same calculation procedure as the ordinary computations is ensured. All types of computations between data shares $\langle x\rangle^m$ and $\langle y\rangle^m$ on all participants $P_m$ will lead to corresponding arithmetic computation result based on $x$ and $y$. Thus all the computation procedures in the secret sharing framework are the same as corresponding arithmetic computations so that the lossless property is guaranteed.
\end{proof}

\section{MP-FedXGB}\label{sec4}
With the preliminaries above, we are now ready to introduce our secure federated XGB framework, MP-FedXGB. As shown in Fig.~\ref{fig:framework}. where we take four participants as an example, all participants exchange intermediate data shares (rather than raw data) to others and then each participant performs computations simultaneously to build the tree for XGBoost. Note that, if $P_m$ has split information in one node, then this node is colored. Otherwise, it is grayed. Finally, every participant can generate a partial tree model, which is a portion of the complete tree of MP-FedXGB. In what follows, we elaborate our proposed MP-FedXGB by dividing it into algorithms of training, tree construction, leaf weight computation, and prediction, along with some sub-functions related to tree construction.

\begin{figure}[htbp]
    \centering
    \includegraphics[width=9cm]{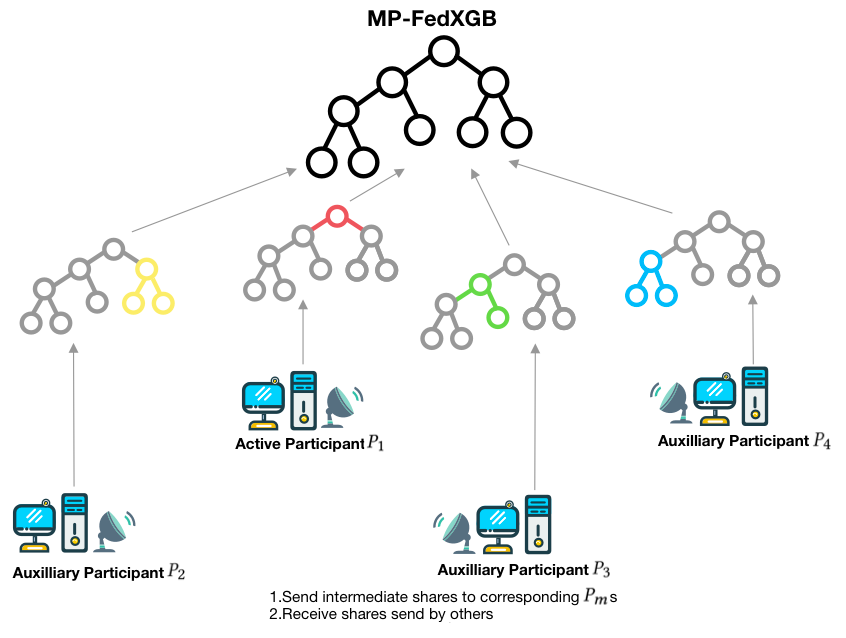}
    \caption{MP-FedXGB Framework.}
    \label{fig:framework}
\end{figure}

\subsection{Secure Training framework}\label{sec4.1}
We start from the training algorithm as shown in Algorithm~\ref{alg:train}. This algorithm will execute a $T$-round for-loop to generate $T$ trees (line 5-11). The input $\boldsymbol{X}$ represents $\boldsymbol{X}_m$ for each $P_m$, while $\boldsymbol{y}$ and $\hat{\boldsymbol{y}}$ are only reachable by $P_1$, the active participant. Others just take two placeholders as input. \textbf{SecureFit} and \textbf{SecurePredict}, detailed later, are invoked $T$ times on every $P_m$ simultaneously to build $T$ trees with label vector $\boldsymbol{y}$ and to produce prediction vector $\hat{\boldsymbol{y}}_t$ at the $t$-th iteration. Only $P_1$ holds the valid prediction $\hat{\boldsymbol{y}}$ for the current tree and updates it (line 9). Finally, a model with $T$ trees is built.
\begin{algorithm}[h!]
   \caption{Secure Training Framework}
   \label{alg:train}
\begin{algorithmic}[1]
    \STATE \textbf{function} \textbf{SecureTrain}($\boldsymbol{X}$,$\boldsymbol{y}$,$\hat{\boldsymbol{y}}$,$T$)
    \STATE // Inputs: instance set $\boldsymbol{X}_m$ for each $P_m$; label $\boldsymbol{y}$ and prediction $\hat{\boldsymbol{y}}$ from $P_1$; tree number $T$.
    \STATE $Trees=[]$ // Initialize empty tree list.
    \STATE $\hat{\boldsymbol{y}}\leftarrow \mathbf{0}_{N\times1}$
    \FOR{$t=1,2,...,T$}
    \STATE $tree_t\leftarrow\textbf{SecureFit}(\mathbf{X},\boldsymbol{y},\hat{\boldsymbol{y}})$ // Generate $T$ trees.
    \STATE $\hat{\boldsymbol{y}}_t\leftarrow tree_t.\textbf{SecurePredict}(\boldsymbol{X})$
    \IF{in $P_1$}
    \STATE $\hat{\boldsymbol{y}}\leftarrow\hat{\boldsymbol{y}}+\hat{\boldsymbol{y}}_t$ // Only $P_1$ have valid predictions, thus it updates final predictions.
    \ENDIF
    \STATE $Trees.append(tree_t)$
    \ENDFOR
    \STATE \textbf{return} $Trees$
\end{algorithmic}
\end{algorithm}

\subsection{Secure Tree Construction}\label{sec4.2}
In beginning, every participant executes Algorithm~\ref{alg:fit} for data preparation. $P_1$ computes necessary variables like first-order gradient vector $\boldsymbol{g}$ and second-order gradient vector $\boldsymbol{h}$, and shares secret slices using \textbf{SHR} to others (line 4-7). What's more, the instance index is another sensitive variable. In consideration of stricter security, as in \cite{fang2020hybrid}, we employ an indicator vector $\boldsymbol{s}_{N\times 1}$ to indicate the location of instances, which composed of only 0 and 1. If one instance belongs to the node, then its corresponding element in $\boldsymbol{s}$ of this node is 1, else 0. After that, the inner product is conducted between the indicator vector and gradient vectors to get a sum of the derivatives of instances in this node. Besides, element-wise multiplication between the indicator vector of one parent node and that of its child node produces the indicator vector for this child to represent which instance belongs to it after the split. Also, indicator vector $\boldsymbol{s}$ can be transformed to secret shares. 

After initialization, a partial tree model is built by recursively executing \textbf{SecureBuild}, shown in Algorithm~\ref{alg:build}. First, each participant computes the sum of derivatives of instances belonging to the current node (line 3). Then they all perform a termination check to examine the depth of the tree (line 4-7). If the maximum depth is reached, then the current node in the partial tree of each $P_m$ is set as a leaf node, and \textbf{SecureLeafWeight} is conducted to calculate leaf weights securely (which will be introduced in Section~\ref{sec4.4} in more details). Else participants will record objective loss in the current node with $\langle loss_n\rangle$ and $\langle loss_d\rangle$, which represent shares of nominator and denominator of $\frac{(\sum_{i\in I}g_i)^2}{\sum_{i\in I}h_i+\lambda}$ from Eq.~\eqref{equation3} respectively (line 8). 

After that, the model will traverse every value of every feature $j$, finding the split candidate with the biggest loss reduction. To speed up this process, we employ \textbf{SecureAggBucket} for every feature to aggregate derivatives into buckets (line 9). This method is clearly demonstrated in \cite{fang2020hybrid} as $SSSumBucket$. In other words, it's equal frequency binning for different features. With these split buckets, participants enumerate all possible splits of their features and choose the one with the highest loss reduction (line 10-18). Specifically, they finish this process by recording intermediate values respectively in matrix $\boldsymbol{\mathcal{G}}$ and $\boldsymbol{\mathcal{H}}$ with different subscripts. Subscripts $L$ and $R$ represent the left and right instance space. Also, these matrices records share values, so they are denoted as $\langle\boldsymbol{\mathcal{G}}_L\rangle$, $\langle\boldsymbol{\mathcal{G}}_R\rangle$, $\langle\boldsymbol{\mathcal{H}}_L\rangle$, and $\langle\boldsymbol{\mathcal{H}}_R\rangle$ respectively. Participants don't calculate Eq.~\eqref{equation3} to get $\mathcal{L}_{split}$ directly because fulfilling \emph{division} is of high complexity under the SS scheme (More detailed explanation will be presented in Section~\ref{sec4.3} and Section~\ref{sec6.1}). With this information, the feature index $j^*$ and the bucket index $k^*$ of the biggest (or best) $\mathcal{L}_{split}$ are generated by \textbf{SecureArgmax}, detailed in Section~\ref{sec4.3}. It also generates the sign of $\mathcal{L}_{split}$, meaning whether this split decreases the objective function or not. 

If the split candidate reduces the loss, the participant who holds $j^*$ sets corresponding elements of the indicator vector for the left or right child to 1 if instances belong to the left or right split, else 0 (line 21-31). The threshold value is denoted as $\mathcal{Q}_{j^*, k^*}$ (line 22), which is generated by \textbf{SecureAggBucket} and represents the feature value in the $k^*$-th split bucket of the $j^*$-th feature. Both indicators will be \textbf{SHR} to other participants respectively. By receiving these, all participants generate their indicator vectors for both branches using element-wise \textbf{MUL} with the parent indicator vector, the same for the first-order gradient and second-order gradient vectors (line 32-34). After that, every participant calls \textbf{SecureBuild} recursively to build its left or right branches. For the participant who has feature $j^*$, it will maintain the split info in its tree node, while others just record it as ``Dummy'' (line 37-40). If the best split doesn't reduce the objective loss value, $P_m$ will set the current node as a leaf too. \textbf{SecureLeafWeight} is called to get the leaf weights again. Tree nodes in partial trees of different $P_m$s will be generated synchronously, preserving the same tree structure in every participant.
\begin{algorithm}[htbp]
   \caption{Fit the data}
   \label{alg:fit}
\begin{algorithmic}[1]
   \STATE \textbf{function} \textbf{SecureFit}($\boldsymbol{X}$,$\boldsymbol{y}$,$\hat{\boldsymbol{y}}$)
   \STATE // Inputs: instance set $\boldsymbol{X}_m$ for each $P_m$; label $\boldsymbol{y}$ and prediction $\hat{\boldsymbol{y}}$ from $P_1$.
   \STATE $\boldsymbol{s}\leftarrow\mathbf{1}_{N\times1}$ // Generate the initial indicator vector.
   \IF{in $P_1$}
   \STATE $\boldsymbol{g},\boldsymbol{h}\leftarrow\frac{\partial l(\boldsymbol{y},\hat{\boldsymbol{y}})}{\partial \hat{\boldsymbol{y}}},\frac{\partial^2 l(\boldsymbol{y},\hat{\boldsymbol{y}})}{\partial \hat{\boldsymbol{y}}^2}$
   \STATE $\langle\boldsymbol{g}\rangle,\langle\boldsymbol{h}\rangle,\langle\boldsymbol{s}\rangle\leftarrow \textbf{SHR}(\boldsymbol{g}),\textbf{SHR}(\boldsymbol{h}),\textbf{SHR}(\boldsymbol{s})$
   \ENDIF
   \STATE $\textbf{Tree}\leftarrow\textbf{SecureBuild}(\langle\boldsymbol{g}\rangle,\langle\boldsymbol{h}\rangle,\langle\boldsymbol{s}\rangle)$
   \STATE \textbf{return} \textbf{Tree}
\end{algorithmic}
\end{algorithm}

\begin{algorithm}[htbp]
  \caption{Secure Tree Building}
  \label{alg:build}
\begin{algorithmic}[1]
  \STATE \textbf{function} \textbf{SecureBuild}($\langle\boldsymbol{g}\rangle$,$\langle\boldsymbol{h}\rangle$,$\langle\boldsymbol{s}\rangle$,$\boldsymbol{X}$)
  \STATE // Inputs: gradient vector $\langle\boldsymbol{g}\rangle^m$, hessian vector $\langle\boldsymbol{h}\rangle^m$, indicator vector $\langle\boldsymbol{s}\rangle^m$, and instance set $\boldsymbol{X}_m$ for each $P_m$; 
  \STATE $\langle g^\Sigma\rangle,\langle h^\Sigma\rangle\leftarrow\sum_n^N\langle\boldsymbol{g}\rangle_n,\sum_n^N\langle\boldsymbol{h}\rangle_n$ // Sum up elements within the vector.
  \IF{max depth reached}
  \STATE $\langle w\rangle\leftarrow \textbf{SecureLeafWeight}(\langle g^\Sigma\rangle,\langle h^\Sigma\rangle)$
  \STATE \textbf{return} $\textbf{Tree}(\langle w\rangle)$
  \ENDIF
  \STATE $\langle loss_n\rangle,\langle loss_d\rangle\leftarrow \langle g^\Sigma\rangle\otimes\langle g^\Sigma\rangle,\langle h^\Sigma\rangle+\langle\lambda\rangle$
  \STATE $\langle\boldsymbol{G}\rangle,\langle\boldsymbol{H}\rangle\leftarrow \textbf{SecureAggBucket}(\langle\boldsymbol{g}\rangle,\langle\boldsymbol{h}\rangle,\boldsymbol{X})$ // Aggregated statistics share $\langle\boldsymbol{G}\rangle_{J\times K}$, $\langle\boldsymbol{H}\rangle_{J\times K}$
  \FOR{$j=1,2,...,J$}
  \STATE $\langle g_L^\Sigma\rangle, \langle h_L^\Sigma\rangle\leftarrow 0,0$
  \FOR{$k=1,2,...,K$}
  \STATE $\langle g_L^\Sigma\rangle,\langle h_L^\Sigma\rangle\leftarrow \langle g_L^\Sigma\rangle+\langle\boldsymbol{G}\rangle_{j,k},\langle h_L^\Sigma\rangle+\langle\boldsymbol{H}\rangle_{j,k}$
  \STATE $\langle g_R^\Sigma\rangle,\langle h_R^\Sigma\rangle\leftarrow \langle g^\Sigma\rangle-\langle g_L^\Sigma\rangle,\langle h^\Sigma\rangle-\langle h_L^\Sigma\rangle$
  \STATE $\langle\boldsymbol{\mathcal{G}}_L\rangle_{j,k},\langle \boldsymbol{\mathcal{G}}_R\rangle_{j,k}\leftarrow\langle g_L^\Sigma\rangle\otimes\langle g_L^\Sigma\rangle,\langle g_R^\Sigma\rangle\otimes\langle g_R^\Sigma\rangle$
  \STATE $\langle\boldsymbol{\mathcal{H}}_L\rangle_{j,k},\langle \boldsymbol{\mathcal{H}}_R\rangle_{j,k}\leftarrow \langle h_L^\Sigma\rangle+\langle\lambda\rangle,\langle
  h_R^\Sigma\rangle+\langle\lambda\rangle$
  \ENDFOR
  \ENDFOR
  \STATE $j^*,k^*,sign \leftarrow \textbf{SecureArgmax}(\langle\boldsymbol{\mathcal{G}}_L\rangle,\langle\boldsymbol{\mathcal{G}}_R\rangle,\langle\boldsymbol{\mathcal{H}}_L\rangle,\langle\boldsymbol{\mathcal{H}}_R\rangle,\langle loss_n\rangle,\langle loss_d\rangle)$
  \IF{$sign$ is positive}
  \IF{$P_{m'}$ have feature $j^*$}
  \STATE $val\leftarrow\boldsymbol{\mathcal{Q}}^{m'}_{j^*,k^*}$//Quantile $\boldsymbol{\mathcal{Q}}^{m'}_{j^*,k^*}$ records the $k^*$-th split bucket value for $j^*$.
  \STATE $\boldsymbol{s}_L\leftarrow\mathbf{1}_{N\times 1}$
  \FOR{$n=1,...,N$}
  \IF{$\mathbf{X}_{n,j^*}>val$}
  \STATE ${\boldsymbol{s}_L}_{n}\leftarrow0$ // Set the $n$-th element to 0.
  \ENDIF
  \ENDFOR
  \STATE $\boldsymbol{s}_R\leftarrow\mathbf{1}_{N\times 1}-\boldsymbol{s}_L$
  \STATE $\langle\boldsymbol{s}^L\rangle,\langle\boldsymbol{s}^R\rangle\leftarrow \textbf{SHR}(\boldsymbol{s}^L),\textbf{SHR}(\boldsymbol{s}^R)$ // Share the local indicator vectors to others.
  \ENDIF
  \STATE $\langle\boldsymbol{s}_L\rangle,\langle\boldsymbol{s}_R\rangle\leftarrow \langle\boldsymbol{s}_L\rangle\otimes\langle\boldsymbol{s}\rangle,\langle\boldsymbol{s}_R\rangle\otimes\langle\boldsymbol{s}\rangle$
  \STATE $\langle\boldsymbol{g}_L\rangle,\langle\boldsymbol{g}_R\rangle\leftarrow \langle\boldsymbol{g}\rangle\otimes\langle\boldsymbol{s}_L\rangle,\langle\boldsymbol{g}\rangle\otimes\langle\boldsymbol{s}_R\rangle$
  \STATE $\langle\boldsymbol{h}_L\rangle,\langle\boldsymbol{h}_R\rangle\leftarrow \langle\boldsymbol{h}\rangle\otimes\langle\boldsymbol{s}_L\rangle,\langle\boldsymbol{h}\rangle\otimes\langle\boldsymbol{s}_R\rangle$
  \STATE $branch_{L}\leftarrow \textbf{SecureBuild}(\langle\boldsymbol{g}_L\rangle,\langle\boldsymbol{h}_L\rangle,\langle\boldsymbol{s}_L\rangle,\boldsymbol{X})$
  \STATE $branch_{R}\leftarrow \textbf{SecureBuild}(\langle\boldsymbol{g}_R\rangle,\langle\boldsymbol{h}_R\rangle,\langle\boldsymbol{s}_R\rangle,\boldsymbol{X})$
  \IF{$P_{m'}$ have feature $j^*$}
  \STATE \textbf{return} $\textbf{Tree}(val,j^*,branch_{L},branch_{R})$
  \ELSE \STATE \textbf{return} $\textbf{Tree}(branch_{L},branch_{R},Dummy)$
  \ENDIF
  \ELSE \STATE $\langle w\rangle\leftarrow \textbf{SecureLeafWeight}(\langle g^\Sigma\rangle,\langle h^\Sigma\rangle)$
  \STATE \textbf{return} $\textbf{Tree}(\langle w\rangle)$
  \ENDIF
\end{algorithmic}
\end{algorithm}

\subsection{Split Candidate Selection}\label{sec4.3}
In the tree learning phase, it is crucial to find the split with the maximum loss reduction expressed in Eq.\eqref{equation3}. Let $\{L_1,R_1\}$ and $\{L_2,R_2\}$ be any two split schema of set $I$, generated by two different split candidates. Furthermore, let us define $\mathcal{G}_X\triangleq(\sum_{i\in X}g_i)^2$ and  $\mathcal{H}_X\triangleq\sum_{i\in X}h_i+\lambda$, $\forall X\in\{I, L_1, L_2, R_1, R_2\}$. Then, with the above definitions, the loss reduction corresponding to the fist split schema $\{L_1,R_1\}$ becomes $\mathcal{L}_1=\frac{1}{2}\left(\frac{\mathcal{G}_{L_1}}{\mathcal{H}_{L_1}}+\frac{\mathcal{G}_{R_1}}{\mathcal{H}_{R_1}}-\frac{\mathcal{G}_I}{\mathcal{H}_I}\right)-\gamma$, while for the second candidate split schema $\{L_2,R_2\}$, the loss reduction after split becomes $\mathcal{L}_2=\frac{1}{2}\left(\frac{\mathcal{G}_{L_2}}{\mathcal{H}_{L_2}}+\frac{\mathcal{G}_{R_2}}{\mathcal{H}_{R_2}}-\frac{\mathcal{G}_I}{\mathcal{H}_I}\right)-\gamma$. Vanilla XGBoost computes the loss reduction of every split candidates and applies \emph{argmax} on them to find the the relation between every $\mathcal{L}_1$ and $\mathcal{L}_2$. In SS, loss reductions are calculated as data shares, e.g., $\langle\mathcal{L}_1\rangle$ and $\langle\mathcal{L}_2\rangle$. Unfortunately, they cannot be restored to $\mathcal{L}_1$ and $\mathcal{L}_2$ to finish the comparison directly due to security concerns. Fang et al. \cite{fang2020hybrid} reconstruct \emph{argmax} operation as the judgment on whether the difference $\mathcal{L}_{diff}=\mathcal{L}_1-\mathcal{L}_2$ is greater than 0. Thus in each participant $P_m(m=1,2)$, $P_m$ computes $\langle \mathcal{L}_{diff}\rangle^m=\langle\mathcal{L}_1\rangle^m-\langle \mathcal{L}_2\rangle^m$. For these two participants, comparison between binary number of $\langle \mathcal{L}_{diff}\rangle^1$ and $\langle \mathcal{L}_{diff}\rangle^2$ bit-by-bit can determine the sign of $\mathcal{L}_{diff}$ without restoring it. Thus they determine the relation between $\mathcal{L}_1$ and $\mathcal{L}_2$ securely.

However, calculating $\langle\mathcal{L}_1\rangle$ and $\langle\mathcal{L}_2\rangle$ needs \emph{division} operations on the SS data shares. There is no direct \emph{division} operation under the SS scheme. \cite{fang2020hybrid} uses a combination of \textbf{ADD} and \textbf{MUL} to approximate \emph{division}, but the calculation overhead is not tolerable, which we will present an analysis in Section~\ref{sec6.1}. In addition, each loss reduction contains more than one fraction, which increases complexity further. Besides, implementation in \cite{fang2020hybrid} conducts the comparison of $\langle\mathcal{L}_{diff}\rangle^1$ and $\langle \mathcal{L}_{diff}\rangle^2$ by multiplexers, and doesn't provide solution or analysis for the multi-participant scenario. From above we find two problems: (1) Calculating loss reduction directly is of high computation complexity; (2) Existing work under the SS scheme is not flexible for multi-participant. Hence we scrutinize the split candidate selection process and propose a novel computation workflow named \textbf{SecureArgmax}. It combines the calculation of loss reductions and \emph{argmax} together to remove \emph{division} operations completely, which simplifies split candidate selection. Additionally, our new design can be applied to more than three participants.
 
To illustrate, the difference of $\mathcal{L}_1$ and $\mathcal{L}_2$ is expressed as:
\begin{equation}\label{equation6}
    \mathcal{L}_{diff}=\mathcal{L}_1-\mathcal{L}_2=\frac{1}{2}\left(\frac{\mathcal{G}_{L_1}}{\mathcal{H}_{L_1}}+\frac{\mathcal{G}_{R_1}}{\mathcal{H}_{R_1}}-\frac{\mathcal{G}_{L_2}}{\mathcal{H}_{L_2}}-\frac{\mathcal{G}_{R_2}}{\mathcal{H}_{R_2}}\right).
\end{equation}
Thus, the relation between $\mathcal{L}_1$ and $\mathcal{L}_2$ is determined via the sign of $\mathcal{L}_{diff}$. By reduction of fractions to the common denominator, it can be written as:
\begin{equation}\label{equation7}
    2\mathcal{L}_{diff}=\frac{\mathcal{H}_{R_1}\mathcal{H}_{R_2}(\mathcal{G}_{L_1}\mathcal{H}_{L_2}-\mathcal{G}_{L_2}\mathcal{H}_{L_1})+\mathcal{H}_{L_1}\mathcal{H}_{L_2}(\mathcal{G}_{R_1}\mathcal{H}_{R_2}-\mathcal{G}_{R_2}\mathcal{H}_{R_1})}{\mathcal{H}_{L_1}\mathcal{H}_{L_2}\mathcal{H}_{R_1}\mathcal{H}_{R_2}}.
\end{equation}
For simplicity, we denote the numerator on the right side of Eq.~\eqref{equation7} as $\mathcal{G}$, while the denominator is represented by $\mathcal{H}$. Thus we have $2\mathcal{L}_{diff}=\frac{\mathcal{G}}{\mathcal{H}}$ as the judgment expression. To obtain the sign of $\mathcal{L}_{diff}$, MP-FedXGB only needs to determine the sign of $\mathcal{G}$ and $\mathcal{H}$ respectively. Every participant computes the share of $\mathcal{H}$ as $\langle\mathcal{H}\rangle=\left(\langle\mathcal{H}_{L_1}\rangle\otimes \langle\mathcal{H}_{L_2}\rangle\right)\otimes\left(\langle\mathcal{H}_{R_1}\rangle\otimes \langle\mathcal{H}_{R_2}\rangle\right)$, similarly for $\mathcal{G}$. To clarify, $\langle\mathcal{G}_{Y_1}\rangle$ and $\langle\mathcal{G}_{Y_2}\rangle$ correspond to any different elements of $\langle \boldsymbol{\mathcal{G}}_Y\rangle$, while $\langle\mathcal{H}_{Y_1}\rangle$ and $\langle\mathcal{H}_{Y_2}\rangle$ correspond to any different elements of$\langle \boldsymbol{\mathcal{H}}_Y\rangle$, $\forall Y\in\{L,R\}$. $\langle \boldsymbol{\mathcal{G}}_Y\rangle$ and $\langle \boldsymbol{\mathcal{H}}_Y\rangle$ are matrices defined in Section~\ref{sec4.2} and Algorithm~\ref{alg:build}. Restoration of $\mathcal{H}$ and $\mathcal{G}$ can determine the sign of these two terms respectively, further determine the sign of $\frac{\mathcal{G}}{\mathcal{H}}$, which is identical to $\mathcal{L}_{diff}$. In our design, $\mathcal{H}$ is restored on $P_1$ and $\mathcal{G}$ is restored on $P_2$. $P_1$ and $P_2$ judge the sign of these two terms respectively. Then $P_2$ sends the judgment to $P_1$. Finally $P_1$ determines the sign of $\frac{\mathcal{G}}{\mathcal{H}}$ and broadcasts it to all $P_m$s, therefore the relation between $\mathcal{L}_1$ and $\mathcal{L}_2$ is determined. This process is shown in Fig.~\ref{fig:sign}. Executing such process between all split candidates will select out the split candidate with the biggest loss reduction $\mathcal{L}_{split}^*$ without calculating Eq.~\eqref{equation7} directly. Besides, the sign of $\mathcal{L}_{split}^*$ is also important. Negative $\mathcal{L}_{split}^*$ indicates the current node as a leaf node, which doesn't split further. Similarly, by reduction of fractions, $\mathcal{L}_{split}^*$ can be expressed as one fraction. Then the same process as determining the sign of $\mathcal{L}_{diff}$ is conducted to obtain the sign of $\mathcal{L}_{split}^*$.

In our designed process, although restored $\mathcal{H}$ and $\mathcal{G}$ are known by $P_1$ and $P_2$ respectively, they won't exchange these data owing to semi-honest premise. Consequently, our \textbf{SecureArgmax} prevents participants from any access to the original value of $\mathcal{L}_{diff}$. Moreover, $\mathcal{H}$ and $\mathcal{G}$ are confused thoroughly, knowing one of them is not sufficient for inferring original information. Therefore leakage of sensitive data is avoided. 

\begin{figure}[htbp]
    \centering
    \includegraphics[width=7cm]{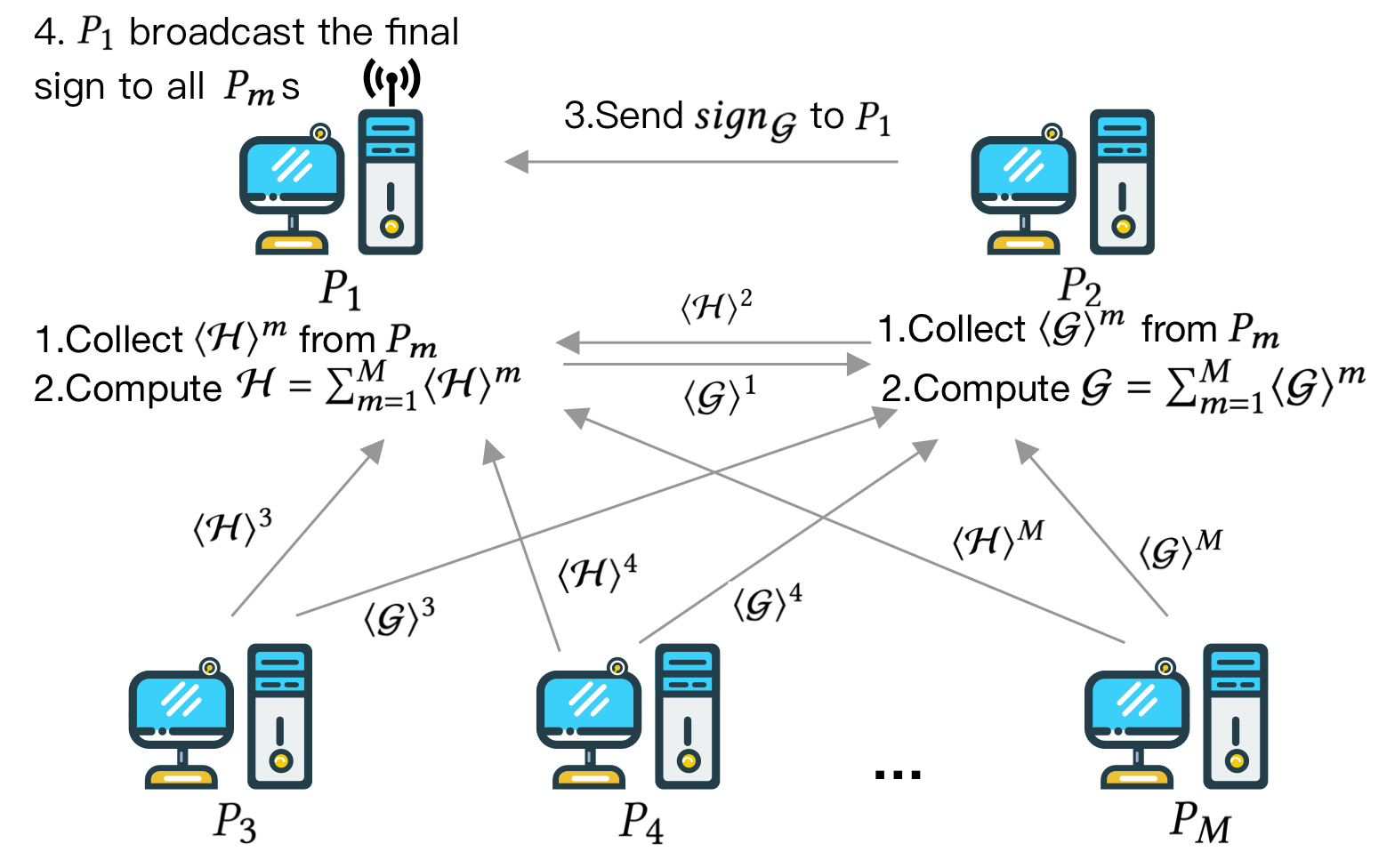}
    \caption{Process on determining the sign of one fraction.}
    \label{fig:sign}
\end{figure}

\subsection{Leaf Weight Computation}\label{sec4.4}
Calculating leaf weight in a distributed way is another crucial problem for training MP-FedXGB. We name this procedure as \textbf{SecureLeafWeight}. XGBoost computes leaf weight as $w=-\frac{\sum_{i\in I_u}g_i}{\sum_{i\in I_u}h_i+\lambda}$ from Eq.~\eqref{equation4}, which needs \emph{division}. As mentioned above, \emph{division} is not defined in SS, hence calculation of fractions is difficult with shared data. Though it can be approximated with addition, subtraction, and multiplication, the difficulty of choosing the proper initial point to converge is not negligible. Meanwhile, the huge amount of iterations would be still time-consuming. 

In fact, leaf weights computation can be viewed as a distributed optimization problem with special local data. In SS every participant $P_m$ holds $\langle\sum_{i\in I_u}h_i+\lambda\rangle^m$ and $\langle\sum_{i\in I_u}g_i\rangle^m$ in its local storage, and these data can recover $a=\sum_{i\in I_u}h_i+\lambda$ and $b=\sum_{i\in I_u}g_i$ according to the definition of \textbf{SHR}. If we denote $\langle a\rangle^m=\langle\sum_{i\in I_u}h_i+\lambda\rangle^m$ and $\langle b\rangle^m=\langle\sum_{i\in I_u}g_i\rangle^m$, the problem of calculating $w=-\frac{\sum_{i\in I_u}g_i}{\sum_{i\in I_u}h_i+\lambda}$ is equivalent to the following quadratic optimization problem
\begin{equation}\label{equation8}
    \arg\min_w\frac{1}{2}\sum_{m=1}^M\langle a\rangle^mw^2+\sum_{m=1}^M\langle b\rangle^mw.
\end{equation}
Because of the requirement of convex loss function in XGBoost, $\sum_{m=1}^M\langle a\rangle^m=\sum_{i\in I_u}h_i+\lambda$ is positive for any $\lambda>0$. The above problem is strongly convex and can be solved by gradient descent method efficiently. By choosing descent step-size $\eta=\frac{1}{\sum_{m=1}^M\langle a\rangle^m}$, we adopt the update of $w$ by
\begin{equation}\label{equation9}
    w^{(1)}=w^{(0)}-\eta\left(\sum_{m=1}^M\langle a\rangle^mw^{(0)}+\sum_{m=1}^M\langle b\rangle^m\right).
\end{equation}
With $w^{(0)}=0$, it has $w^{(1)}=-\frac{\sum_{m=1}^M\langle b\rangle^m}{\sum_{m=1}^M\langle a\rangle^m}=-\frac{\sum_{i\in I_u}g_i}{\sum_{i\in I_u}h_i+\lambda}$ as the final solution. Thereby only one descent step is needed. For the distributed scenario, each $P_m$ can calculate
\begin{equation}\label{equation10}
    \langle w^{(1)}\rangle^m=\langle w^{(0)}\rangle^m-\eta\left(\langle a\rangle^m\otimes\langle w^{(0)}\rangle^m+\langle b\rangle^m\right).
\end{equation}
A summation of these equations from different $P_m$s is equivalent to Eq.\eqref{equation9} according to the results shown in Theorem~\ref{theorem1} (mentioned in Section~\ref{sec3.3}). The $\langle w^{(1)}\rangle^m$ is stored in a distributed way and can be restored to $w^{(1)}$.

However, if data $a$ is sensitive, its restoration is strictly prohibited, thus the precise step-size $\eta=\frac{1}{\sum_{m=1}^M\langle a\rangle^m}$ is hard to obtain. We figure out another idea by adding a small positive perturbation $\sigma^m$ into each $\langle a\rangle^m$. Adding the perturbation term can effectively mask the value of $a$ during the step-size determination. Next, participant $P_1$ determines the new descent step-size $\eta'=\frac{1}{\sum_{m=1}^M(\sigma^m+\langle a\rangle^m)}$ and broadcasts it to other participants. Because each $\sigma^m$ is positive, $\eta'$ is smaller than $\eta$. Therefore, we only need to change the step-size rather than the problem itself. As long as the problem is strongly convex, a sufficient number of iterations will guarantee that the generated iterates can still achieve the global minimum solution. The descent update can be expressed as
\begin{equation}\label{equation11}
    w^{(t)}=w^{(t-1)}-\eta'\left(\sum_{m=1}^M\langle a\rangle^mw^{(t-1)}+\sum_{m=1}^M\langle b\rangle^m\right).
\end{equation}
Again, every term except $\eta'$ can be split to form the distributed optimization problem as Eq.~\eqref{equation10}. Thus $w^{(t)}$ can be solved and stored in a distributed way. The index $t$ needed for convergence can be predicted by analysis of relation between $a$ and $\sigma^m$ (Please see a detailed example shown in Appendix~\ref{appsec1}). Indeed, the above optimization method of fulfilling \emph{division} operation can be applied to more general cases where there has similar distributed data structures and also privacy concerns.

\subsection{Prediction}\label{sec4.5}
For a partial tree stored in different participants, prediction should be finished by cooperation. Inspired by existing work \cite{fang2020hybrid}, we adopt the similar idea of indicator vector to do the prediction called \textbf{SecurePredict}. We continue our example from Fig~\ref{fig:framework} with 4 participants. For vanilla XGBoost, prediction of the $i$-th instance $\boldsymbol{x}_i$ is expressed as $\hat{y}_i\triangleq\boldsymbol{s}\cdot\boldsymbol{w}$, shown in Fig~\ref{fig:Prediction}. Here ``$\cdot$'' means a dot product. For the distributed weights, we rewrite this equation as
\begin{equation}\label{equation12}
    \hat{y}_i=\boldsymbol{s}\cdot\sum_{m=1}^M\langle\boldsymbol{w}\rangle^m,
\end{equation}
where $\langle\boldsymbol{w}\rangle^m$ denotes the weight vector composed of weight shares produced by \textbf{SecureLeafWeight}.
\begin{figure}[htbp]
    \centering
    \includegraphics[width=10cm]{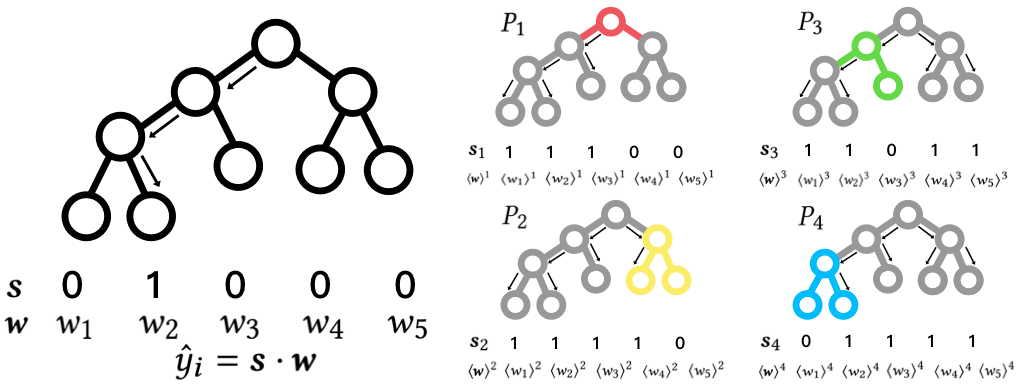}
    \caption{Process on determining the sign of one fraction.}
    \label{fig:Prediction}
\end{figure}

For a partial tree in $P_m$, if one node doesn't have split information, it will record ``Dummy'' during the tree learning stage, marked as gray nodes. The searching process is expressed as black arrows in Fig~\ref{fig:Prediction}. $P_m$ walks along a determined path from its colored nodes, but searches for both left and right branches when facing gray nodes. After the search is complete, the selected leaves are flagged as ``1'' else ``0''. Thus $P_m$ can obtain local indicator vector $\boldsymbol{s}_m$. Notice that $\prod_{m=1}^M\boldsymbol{s}_m=\boldsymbol{s}$, where $\prod$ means element-wise multiplication. Even though ``1'' is uncertain for every leaf, ``0'' is definite from each $P_m$ who has split information. So ``0'' can eliminate redundant ``1''. Actually, this multiplication is the ``AND'' operation on different partial trees and can filter out branches not chosen. Finally, only one leaf is determined. From Fig.~\ref{fig:Prediction} it's obvious that this operation can be finished even for more than two participants, therefore, it is suitable for our problem definition.

To better preserve privacy, participants should produce prediction $\langle\hat{y}_i\rangle$ locally and restore $\hat{y}_i$ on active $P_1$ rather than exchanging local indicator and weight vector directly. For $P_m$, $\boldsymbol{s}_m$ can be shared using $\textbf{SHR}$, thus $\boldsymbol{s}_m=\sum_{m'=1}^M\langle\mathbf{s}_m\rangle^{m'}$. With associativity and definition of \textbf{MUL}, we can rewrite $\prod$ as follows
\begin{equation}\label{equation13}
    \begin{aligned}
     \boldsymbol{s}=\prod_{m=1}^M\boldsymbol{s}_m
     =\prod_{m=1}^M\sum_{m'=1}^M\langle\boldsymbol{s}_m\rangle^{m'}
     =\sum_{m'=1}^M\langle\boldsymbol{s}_1\rangle^{m'}\otimes\langle \boldsymbol{s}_2\rangle^{m'}\otimes\cdots\langle \boldsymbol{s}_M\rangle^{m'}
     =\sum_{m=1}^M\langle\boldsymbol{s}\rangle^m.
    \end{aligned}
\end{equation}
We discover that local secret indicator $\langle \boldsymbol{s}\rangle$ can be obtained by \textbf{MUL} among shares of local indicator vectors from other participants. Substituting Eq.~\eqref{equation13} into Eq.~\eqref{equation12}, we have
\begin{equation}\label{equation14}
    \begin{aligned}
    \hat{y}_i=\sum_{m=1}^M\langle \boldsymbol{s}\rangle^{m}\cdot\sum_{m=1}^M\langle\boldsymbol{w}\rangle^m
    =\mathbf{1}\cdot\sum_{m=1}^M\langle\boldsymbol{s}\rangle^m\otimes\langle\boldsymbol{w}\rangle^m
    =\sum_{m=1}^M\langle\hat{y}_i\rangle^m.
    \end{aligned}
\end{equation}
From the above inference each $P_m$ can obtain partial prediction by calculating \textbf{MUL} on $\langle\boldsymbol{s}\rangle^m$ and $\langle\boldsymbol{w}\rangle^m$. Finally restoration in $P_1$ can generate the prediction $\hat{y}_i$. Hence $P_1$ makes prediction without exchanging local indicator vectors and weight vectors directly. This procedure can be executed in training to get updated value or prediction for instances.

\section{Security Discussion}\label{sec5}
In this section, we discuss the semi-honest premise which ensures the privacy of MP-FedXGB. With this guarantee, our model is secure for raw data $\boldsymbol{X}$ and intermediate values during computation on data shares. Nevertheless, the training framework may pose potential leakage to instance space, which further leaks data labels $\boldsymbol{y}$. We also carry out our analysis and provide possible solution.

\subsection{Semi-Honest Security}\label{sec5.1}
Our proposed model is based on the semi-honest adversarial setting. That is, all participants are honest-but-curious. We assume only the active participant $P_1$ is honest. Up to $M-1$ corrupt participants (all auxiliary participants) might cooperate with each other in order to gather private information. For example, if one participant has a share of data, it may restore sensitive information by gathering shares of corresponding data from other participants. As mentioned, there is no collusion between $P_1$ and any auxiliary participant $P_m(m\neq1)$. Therefore $P_1$ doesn't leak any raw data and intermediate information to other participants, as well as shares of them. Due to lack of data shares from $P_1$, collusion between auxiliary participants cannot infer the original sensitive data. Hence our MP-FedXGB is secure for raw data $\boldsymbol{X}$ and intermediate values under semi-honest assumption.

\subsection{Potential Leakage Towards Instance Space}\label{sec5.2}

Instances belonging to one leaf node have the same instance space, which tends to have the same target value or label. In MP-FedXGB, the indicator vector can effectively mask instance space, because the indicator vector on any node of the tree is generated for the whole instance set rather than a subset from its parent node. Even $P_m$ splits two leaf nodes, its own indicator vector contains split information for all data. The coarse-grained split result reflects by indicator vector of $P_m$ won't help much to leak instance space. Thus the ground truth labels are protected. There will be only one situation where $P_m (m\neq1)$ knows the fine-grained instance space, i.e., there is a direct path from the root to leaf that all the intermediate nodes are generated using information from $P_m$. In this situation, this $P_m$ can filter instance subsets from element-wise multiplications among the intermediate indicators. The simplest situation is that the depth of the tree is 2, and the only parent node split instances using features of $P_m$. 

To better protect the instance space, we add a mechanism to our tree building process called \emph{First-Layer-Mask}. As a consequence, the root node can only be split by $P_1$ for each tree in the ensemble model. In this design, instance space won't leak as $P_1$ knows its instance space while others can only receive a secret indicator after this split. For the next layer, no matter which participant does split, its result is for the whole instance set again. We cut the direct path from the root to leaf on the first layer by enforcing $P_1$ to perform the split, thus we are able to ensure the security of instance space. Experiments are conducted to examine the performance of the model before and after using this mechanism, which is presented in Section~\ref{sec7.3}.

\section{Complexity Analysis}\label{sec6}
In this section, we analyze the complexity of our proposed \textbf{SecureArgmax} against trivial \emph{argmax} with \emph{division} operations. In addition, we also compare the time consumption between methods using HE and our framework.

\subsection{Computation Complexity of SecureArgmax}\label{sec6.1}
Our proposed \textbf{SecureArgmax} computes the difference between loss reductions of different split candidates. Every participant computes its $\langle\mathcal{H}\rangle$ and $\langle\mathcal{G}\rangle$ using local shares. From Eq.~\eqref{equation7} these operations contain nine \textbf{MUL}s in total. Notice that we count the execution number within arbitrary one participant, because all participants execute the same process. Then binary search is conducted that the number of split candidates are halved after each comparison. Here the comparison among different splits is identical to computing difference. Suppose the dataset has $J$ features and $K$ split buckets for each feature. In each feature $\lceil \log_2K\rceil$ comparisons is needed, resulting in $9\lceil \log_2K\rceil$ \textbf{MUL}s. Here $\lceil\cdot\rceil$ means \emph{ceil}. Hence for $J$ features the algorithm computes $9J\lceil \log_2K\rceil$ \textbf{MUL}s. Next, \textbf{SecureArgmax} selects the best split among $J$ features, thus it conducts $\lceil \log_2J\rceil$ comparisons, resulting in $9\lceil \log_2J\rceil$ \textbf{MUL}s. Finally our method needs $9J\lceil \log_2K\rceil+9\lceil \log_2J\rceil$ \textbf{MUL}s in total. It's the key computation of this approach.

To show the efficiency of \textbf{SecureArgmax} better, we also try to simulate the \emph{argmax} with approximation approach of \emph{division} and analyze its number of \textbf{MUL}s. Goldschmidt's approach \cite{goldschmidt1964applications} uses multiplication and addition (or subtraction) to approximate \emph{division}, which is mentioned in \cite{fang2020hybrid}. Unfortunately, no detailed implementation is provided. Instead, we implement a similar approach, Newton's method. They both require iterations and combinations of many computation primitives. We elaborate this method as \textbf{DIV} in Appendix~\ref{appsec2}. During the split candidate selection in one tree node, the full instance set keeps the same. Consequently, $\mathcal{L}_{split}$ from Eq.~\eqref{equation3} can omit the last two terms, thus the expression of each loss reduction contains two different fractions. In our test on a simple Iris dataset \cite{Dua:2019}, the model will converge (prediction results don't change) with at least 20 iterations for our implemented Newton's method to converge, which includes 41 \textbf{MUL}s. Therefore, each loss reduction of different split candidates requires 82 \textbf{MUL}s to converge. For $J$ features and $K$ splits in each feature, there are $JK$ split candidates in total, resulting in $82JK$ \textbf{MUL}s for finishing operation \emph{argmax}.

Different numbers of \textbf{MUL}s taken in calculation for \textbf{SecureArgmax} and the simulated \emph{argmax} are presented in Table~\ref{tab:table2}. Our proposed \textbf{SecureArgmax} is obviously superior in efficiency. Besides, real-world applications may have more features and splits, which is more suitable for our \textbf{SecureArgmax}.

\begin{table}[htbp]
  \caption{Amount of \textbf{MUL}s in different methods}
  \begin{tabular}{ccl}
    \toprule
    Param & Using \textbf{DIV} & No \textbf{DIV} (ours)\\
    \midrule
    J=16, K=8    & 10,496 & 468\\
    J=16, K=16    & 20,992 & 612\\
    J=32, K=16   & 41,984 & 1,197\\
  \bottomrule
  \label{tab:table2}
\end{tabular}
\end{table}

\subsection{Computation Complexity against HE methods}\label{sec6.2}
Suppose HE methods need time $t_1$ to encrypt a data point, and $t_2$ to decrypt a data point. SecureBoost uses the Paillier algorithm to perform encryption and decryption. XGBoost requires the first-order gradient and second-order gradient, thus one data instance corresponds to two encryptions or two decryptions. Suppose there are $J$ features and $K$ split buckets for each feature on average, $\alpha$ fraction of features are stored outside $P_1$. For $N$ instances, encryption of $2N$ gradient information is needed. After bucket aggregation of the gradient information, $2\alpha JK$ decryption is taken to retrieve back aggregated gradient statistics from the auxiliary participants. If $e$ trees are built and each has a depth of $d$, for convenience we set $\alpha=0.5$, then the total time cost is $time_{HE}=2Nt_1+e(2^d-1)JKt_2$.

Besides, our framework mainly consume time on the bucket statistics aggregation, best split finding, and preparation for data to left or right branches. We only take into account the time consumption of \textbf{MUL} because it is the main computation. Our implementation of \textbf{MUL} can execute element-wise multiplication, thus is much faster and executes less in number of times. For bucket statistics aggregation, it executes $2JK$ \textbf{MUL}s. In our implementation, we actually suppress it to $2J$ \textbf{MUL}s by a proper design. For split finding and sign determination, $9J\lceil \log_2K\rceil+9\lceil \log_2J\rceil+8$ \textbf{MUL}s are executed. In preparation of data on child nodes, 6 \textbf{MUL}s are conducted. During training, each prediction update requires element-wise \textbf{MUL}s between $M$ shared local indicator vectors, which have $N$ elements. Thus the prediction needs $MN$ \textbf{MUL}s. $e-1$ rounds of training need the update of prediction, causing $MN(e-1)$ \textbf{MUL}s in total. Suppose each \textbf{MUL} needs $t_3$, in total the time cost of operation inside our framework is $time_{MP-FedXGB}=[e(2^d-1)(2J+9J\lceil \log_2K\rceil+9\lceil \log_2J\rceil+14)+MN(e-1)]t_3$. For executing one python package\footnote{https://github.com/n1analytics/python-paillier} of the Paillier algorithm \cite{1999Public} with a 1024-bit key, it takes $t_1=0.01s$ and $t_2=0.19s$ while our \textbf{MUL} takes $t_3=0.0005s$ on average. Suppose we have $M=4$ participants and $N=10,000$ instances. If we build a model with $e=3$ trees, max depth $d=3$ for each, $J=10$ features and $K=10$ buckets, we will have $time_{HE}=599$s and $time_{MP-FedXGB}=44.52$s. Notice that we don't calculate the time cost of communication because we just simulate our framework in one machine with virtual participants. We aim to show the computational efficiency of our framework.

\section{Experiment}\label{sec7}
In this section, we investigate the efficiency and performance of MP-FedXGB in different model structures. Also, the performance of our model with the proposed \emph{First-Layer-Mask} in Section~\ref{sec5.2} is illustrated. Here, to save space and better explore the performance, all experiments are conducted on public datasets for binary classification tasks but similar observations can be made for regression tasks. 

\textbf{Dataset1}\footnote{https://www.kaggle.com/c/GiveMeSomeCredit/data}: This dataset provides data for exploring whether a user would suffer from financial problems. It contains 150,000 instances and 10 features.

\textbf{Dataset2}\footnote{http://archive.ics.uci.edu/ml/datasets/Adult}: This dataset is prepared for the prediction task to determine whether a person earns over 50K a year or not. It contains 48,842 instances and 14 features. Most of its features are categorical features, we employ one-hot encoding to transform these features. Finally, the processed dataset has 108 features.

These two datasets both have missing values. We conduct our experiments on cleaned datasets. To make the result more plausible, we keep the ratio of 0-1 class the same from the testing set to the training set. Besides, min-max normalization is applied. We take \textbf{Dataset1} for exploration of the relation between different training instance sizes, feature sizes, and runtime. We limit the biggest size of the training instance set to 30,000. In the performance comparison with different models, we take 80\% for training and 20\% for testing. So \textbf{Dataset1} has 30,000 training instances and 7,500 testing data, while \textbf{Dataset2} has 26,049 for training and 6,512 for testing. The default model structure is of three trees, each has a max depth of three. We set 4 participants and 1 coordinator, and partition the instance feature according to the proportion of 10\%, 20\%, 30\%, and 40\% to $P_1$, $P_2$, $P_3$, and $P_4$ respectively. We set $\lambda=1$ and $\gamma=0.5$, and they are equally shared as $\langle \lambda\rangle$ and $\langle \gamma\rangle$ to the 4 participants. The framework is implemented using mpi4py \cite{2011Parallel}, a python package that allows distributed and parallel programming. All experiments are conducted on a Mac with 16GB RAM and Intel Core i9 CPU.

\subsection{Scalability}\label{sec7.1}
Since XGBoost yields better performance when the depth and number of trees grow, we first examine how the computation time of the proposed MP-FedXGB changes with the model structure. The experiment is conducted on \textbf{Dataset1} with 10,000 instances. The results are presented in Fig.~\ref{fig4a} and Fig.~\ref{fig4b}.

From Fig.~\ref{fig4a} we can see that the runtime is approximately linear with the number of trees. However, Fig.~\ref{fig4b} shows that the runtime goes exponentially with respect to the depth of each tree. These are consistent with vanilla XGBoost since the number of tree nodes is proportional to the number of trees, and grows exponentially with respect to the depth. Time-consuming computations mainly exist in split finding, which is conducted for node generation. Besides, the result in Fig.~\ref{fig4a} with three trees and $depth=3$ has a runtime of approximately 50 seconds, which agrees with the analysis of runtime of our model shown in Section~\ref{sec6.2}.

\begin{figure*}[htbp]\label{fig4}
\centering
\subfigure[Number of Trees v.s. runtime]{
\label{fig4a}
\includegraphics[width=6cm]{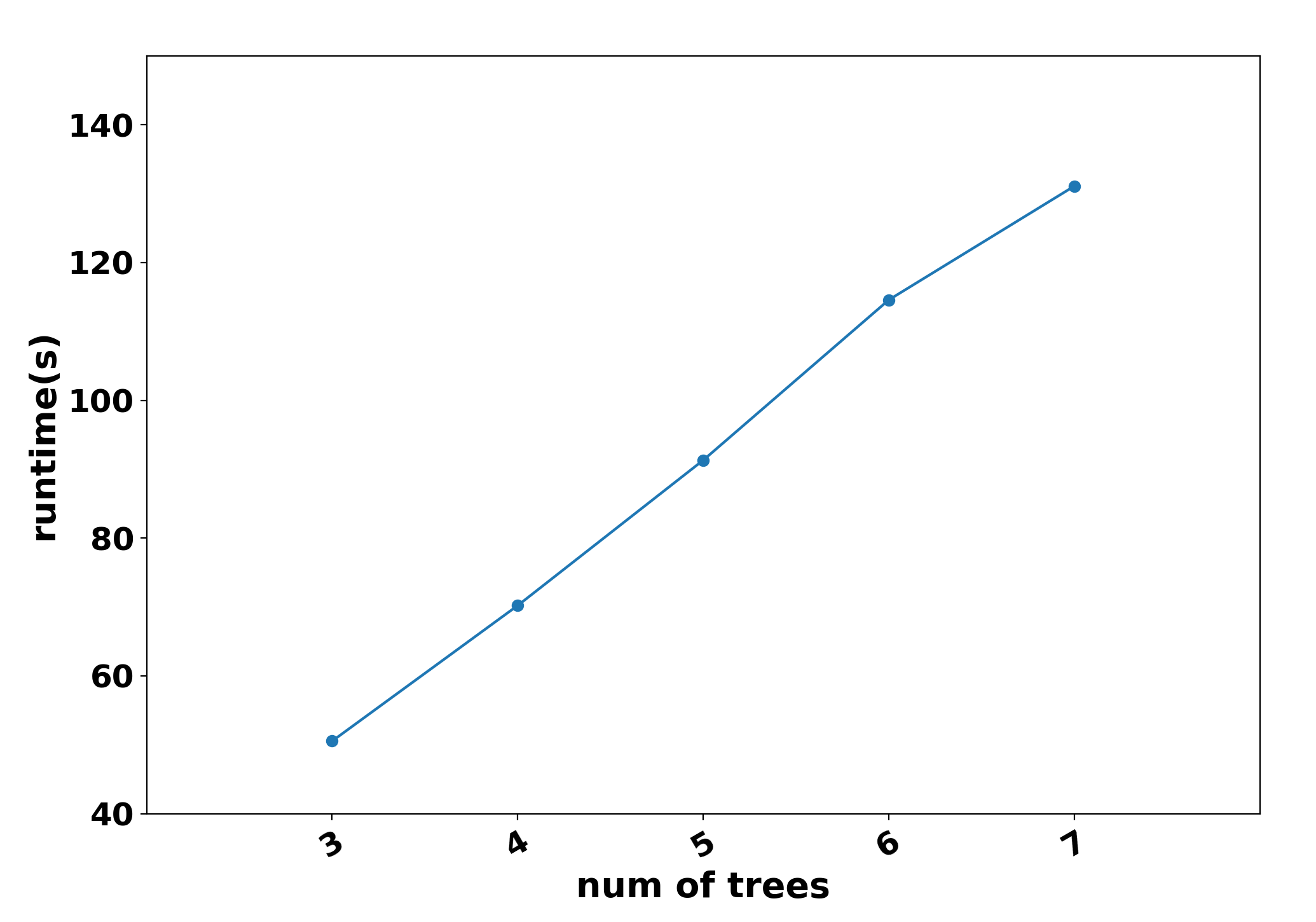}
}
\subfigure[Max depth v.s. runtime]{
\label{fig4b}
\includegraphics[width=6cm]{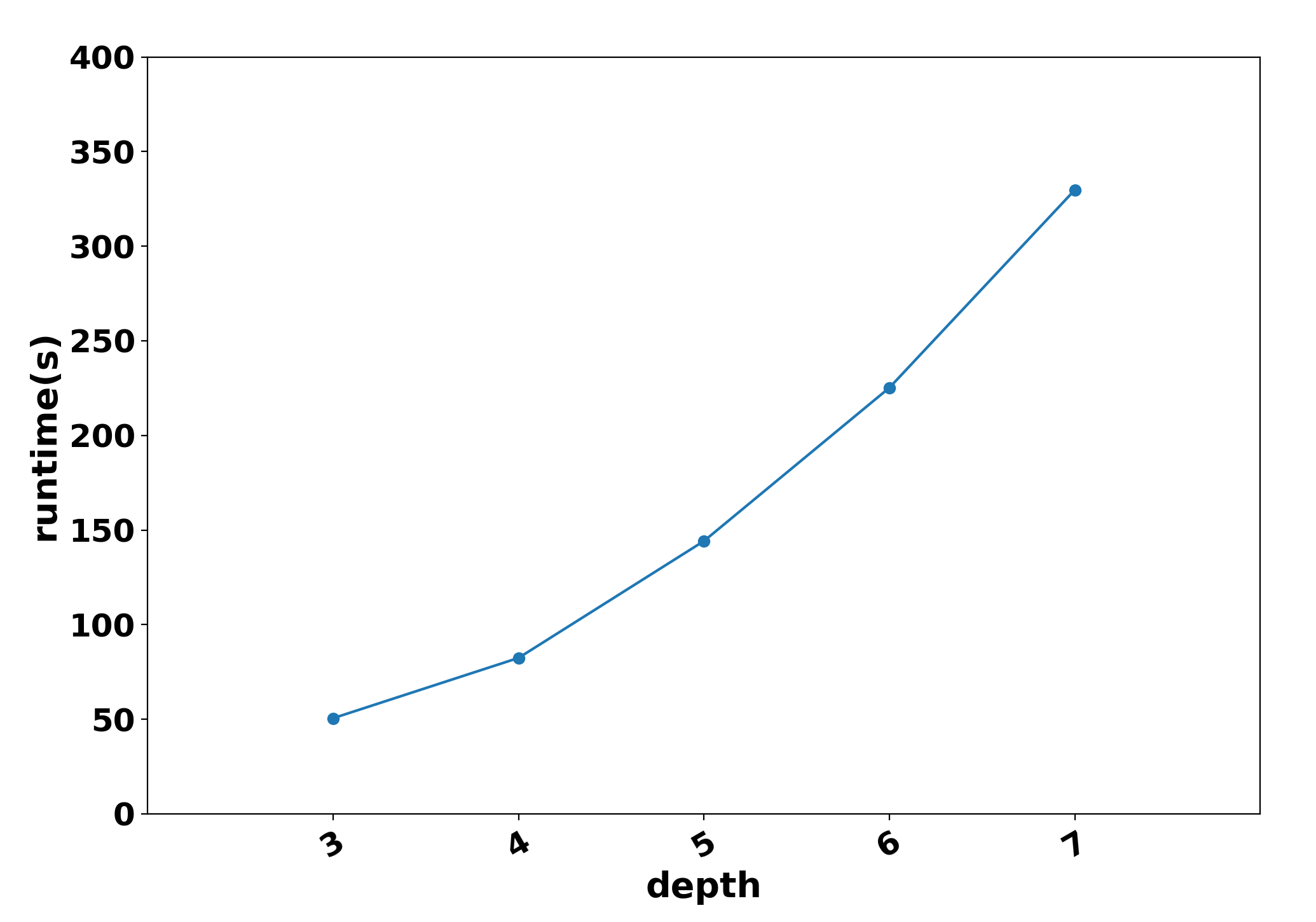}
}
\caption{Runtime in Different Tree Structures, Feature Numbers}
\end{figure*}

In addition, we explore the relationship in terms of different training instance sizes and feature sizes. Specifically, we vary instance size among \{5,000, 10,000, 30,000\} and change feature sizes among \{10, 50, 100, 500, 1,000\}. We augment the feature set by simple repetition of original features. Results in Fig.~\ref{fig:size} show that runtime has a linear relationship with the feature size. In addition, the runtime grows for a larger instance set, which is consistent with the expression of time complexity we provided in Section~\ref{sec6.2}.

\begin{figure}[htbp]
    \centering
    \includegraphics[width=6cm]{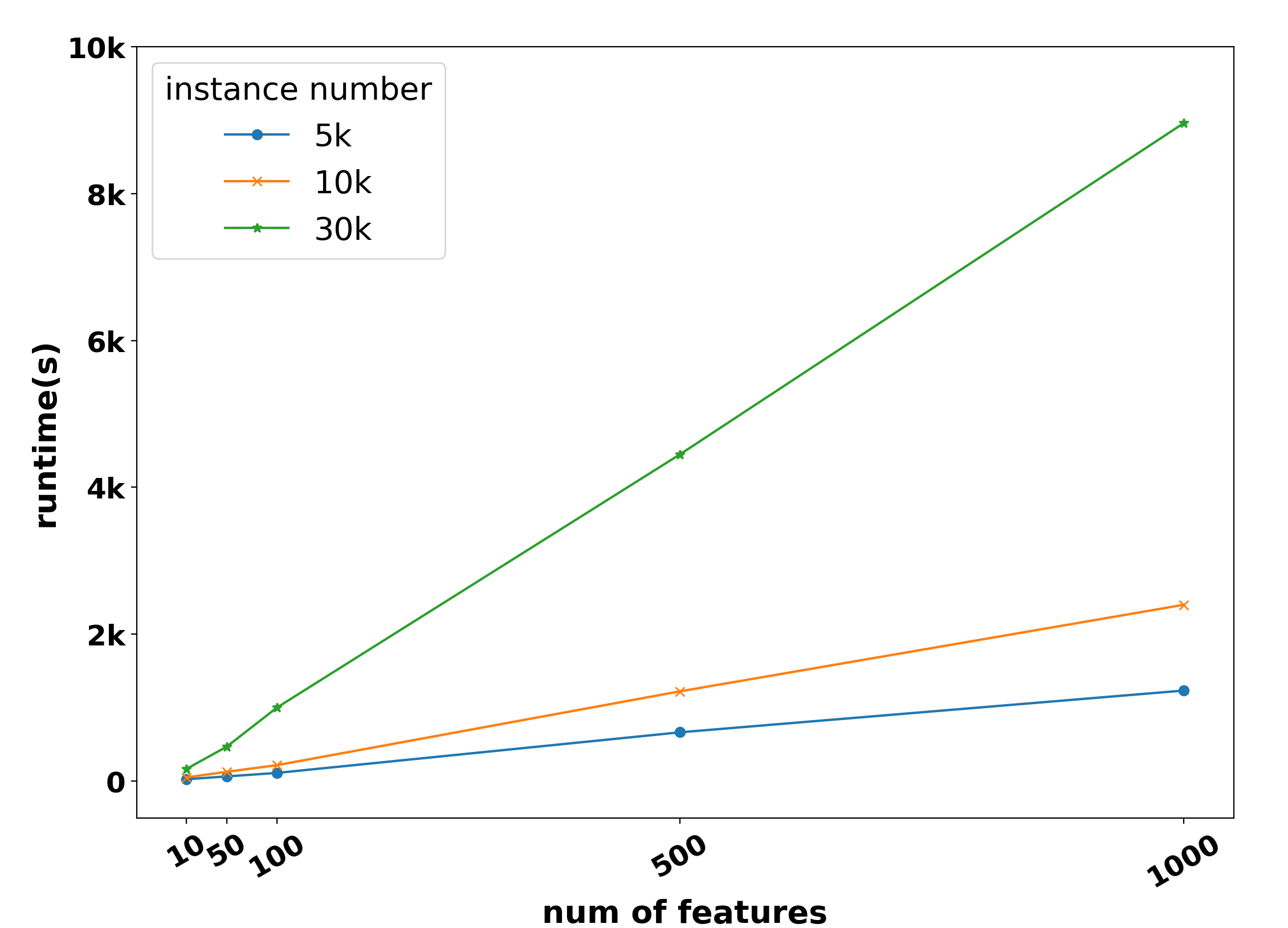}
    \caption{Relationships with respect to different instance sizes, feature sizes and runtime}
    \label{fig:size}
\end{figure}

\subsection{Performance of MP-FedXGB}\label{sec7.2}
To demonstrate the superiority of our proposed MP-FedXGB, we conduct comparisons between our model and vanilla XGBoost\footnote{https://github.com/dmlc/xgboost}. Vanilla XGBoost is trained in a centralized fashion and all hyperparameters are the same as ours. It should be noted that MP-FedXGB we implemented may be slightly different from the original one (like quantile sketch split schema), so all metrics would differ a little, but are exactly the same as our implemented centralized XGBoost. We test the performance of both algorithms with three trees and different depths. As shown in Table~\ref{tab:table3}, MP-FedXGB exhibits comparable performance as the vanilla XGBoost in all metrics, and even better in some situations. This testifies the validity of MP-FedXGB.

\begin{table*}[htb]
\caption{Comparison with respect to Different Models under Different Datasets and Depths}
\label{tab:table3}
\begin{tabular}{c|c|c|c|c|c|c|c|c|c|c}
\hline
\multirow{2}{*}{Dataset} & \multirow{2}{*}{Depth} & \multicolumn{3}{c|}{XGBoost} & \multicolumn{3}{c|}{MP-FedXGB} & \multicolumn{3}{c}{MP-FedXGB*} \\ \cline{3-11}
 &  & ACC & F1 & AUC & ACC & F1 & AUC & ACC & F1 & AUC \\ \hline
\multicolumn{1}{r|}{\multirow{3}{*}{Dataset1}} & 3 & 0.9319 & 0.1354 & 0.8020 & 0.9324 & 0.1914 & 0.8410 & 0.9331 & 0.1824 & 0.8255 \\
\multicolumn{1}{r|}{} & 4 & 0.9316 & 0.1818 & 0.8302 & 0.9321 & 0.2253 & 0.8377 & 0.9317 & 0.2099 & 0.8342 \\
\multicolumn{1}{r|}{} & 5 & 0.9327 & 0.2337 & 0.8313 & 0.9329 & 0.2481 & 0.8379 & 0.9323 & 0.1962 & 0.8397 \\ \hline
\multirow{3}{*}{Dataset2} & 3 & 0.8455 & 0.6192 & 0.8639 & 0.8358 & 0.5965 & 0.8850 & 0.8332 & 0.5954 & 0.8844 \\
 & 4 & 0.8500 & 0.6415 & 0.8844 & 0.8423 & 0.6403 & 0.8940 & 0.8404 & 0.6156 & 0.8936 \\
 & 5 & 0.8561 & 0.6447 & 0.9014 & 0.8449 & 0.6481 & 0.8966 & 0.8424 & 0.6420 & 0.8945 \\ \hline
\end{tabular}
\end{table*}
% \vspace{-0.5cm}

\subsection{Performance Of First Layer Mask}\label{sec7.3}
To verify our \emph{First-Layer-Mask} mechanism, we compare the original model to the model with this mechanism applied to, denoted as MP-FedXGB$^*$. Table~\ref{tab:table3} shows the results with three trees and different depths. Our experiment shows that even we force $P_1$, the participant with the label, to perform a split at the beginning in every tree, the proposed model won't suffer much from performance loss. We infer that XGBoost itself is robust and powerful, so an extra split won't impact the model significantly.

\section{Conclusion}\label{sec8}
In this paper, we have proposed an SS-based \emph{multi-party} federated XGBoost model, which is applicable to the vertical federated learning scheme. We have carefully designed the split finding process and carried out the leaf weight computation via solving a quadratic optimization problem in a distributed way, making our model fast in runtime and superior in performance even facing large-scale training scenarios in real-world applications. 

A secret-sharing scheme is a special encryption method that provides both safety and efficiency. We believe that the research in federated learning with SS is just in its infancy. By extending the basic idea of MP-FedXGB, we hope to generalize our work and develop an efficient and secure vertical federated learning approach for broader machine learning models. Besides, we expect to further reduce the communication overhead and simplify the computation in our future work.

%\bibliography{sample-manuscript}

\appendix
\section{Convergence Analysis of Gradient Descent with Perturbation}\label{appsec1}
To give a clearer analysis, we suppose that our task is a two-class classification problem. Without loss of generality,  we can write the descent iteration in Eq.\eqref{equation11} as the following
\begin{equation}\label{equation15}
\begin{aligned}
w^{(t)}=w^{(t-1)} - \frac{1}{ra}(aw^{(t-1)}+b)=\frac{r-1}{r}w^{(t-1)}-\frac{b}{ra},
\end{aligned}
\end{equation}
where we choose the descent step-size as $\eta'=\frac{1}{ra} (r\ge1)$,  and express $a=\sum_{m=1}^M\langle a\rangle^m$ and $b=\sum_{m=1}^M\langle b\rangle^m$ for simplicity. So, we have $w^{(t)}+\frac{b}{a}=\frac{r-1}{r}(w^{(t-1)}+\frac{b}{a})=...=(\frac{r-1}{r})^t(w^{(0)}+\frac{b}{a})$.  By solving this geometric progression, we can get $w^{(t)}=((\frac{r-1}{r})^t-1)\frac{b}{a}$, where we set the initial solution $w^0$ as 0.

For this special second-order optimization problem, the optimal solution is $w^*=-\frac{b}{a}$. Then we can analyze the iteration complexity through convergence to $\epsilon$-precision. For simplicity, we suppose that $w^*\le 0$, then the following relation holds
\begin{equation}\label{equation16}
\begin{aligned}
|w^{(t)}-w^*|=\left(\frac{r-1}{r}\right)^t\frac{b}{a}\le\epsilon,
\end{aligned}
\end{equation}
which implies
\begin{equation}\label{equation17}
    t\ge\frac{\ln\frac{a\epsilon}{b}}{\ln\frac{r-1}{r}}.
\end{equation}
Since that $r\ge1$, we have $\ln\frac{r-1}{r}<0$. If $\frac{a\epsilon}{b}\ge1$ then $w^{(t)}$ will converge with $t=1$. But in general $\vert w^*\vert=\vert-\frac{b}{a}\vert$ is not an extreme value, so $\frac{a\epsilon}{b}$ will be smaller than 1 easily with small precision factor $\epsilon$. For fixed $r$ and $\frac{a\epsilon}{b}\le1$, $t$ grows when $\frac{a\epsilon}{b}$ is smaller. We are curious about how big $t$ will be when $\frac{a\epsilon}{b}$ is extremely small. We first figure out the inferior limit of $\frac{a\epsilon}{b}$.

For general 0-1 classification problems and log-loss, we assume the dataset has $10^8$ samples that $\vert I_u\vert=10^8$, and the worst case is that all the samples are classified totally wrong. Let's denote $pred_i=\frac{1}{1+e^{-\hat{y}_i}}$, where $\hat{y}_i$ is the predicted label. With calculation of derivatives, $w^*=-\frac{\sum_{i\in I_u}(pred_i-y_i)}{\sum_{i\in I_u}pred_i(1-pred_i)+\lambda}$, where $y_i$ is the ground truth. We assume that all the samples of this leaf are of class 0 ($y_i=0$) and be miss-classified to 1. Suppose that $p\triangleq pred_i$ is the biggest and $\lambda\ge1$, then we have
\begin{equation}\label{equation18}
\begin{aligned}
    w^*=-\frac{\sum_{i\in I_u} pred_i}{\sum_{i\in I_u}pred_i(1-pred_i)+\lambda}
    \ge -\sum_{i\in I_u} pred_i
    \ge -\vert I_u\vert p.
\end{aligned}
\end{equation}
Here $\vert\cdot\vert$ means the cardinal number. Since miss-classified, $p$ can be set to 1, then $-\frac{b}{a}\ge-\vert I_u\vert p=-10^8$. Due to the fact that $\frac{a\epsilon}{b}\ge10^{-8}\epsilon=10^{-14}$ if $\epsilon = 10^{-6}$, the inferior limit of $\frac{a\epsilon}{b}$ can be obtained.

In our setting, each $P_m$ adds a perturbation $\sigma_m$ on its $\langle a\rangle^m$. We require $\sum_{m=1}^M\sigma_m\le \mu\lambda$, where $\lambda$ is the regularization term in Eq.~\eqref{equation2} and $\mu$ is a small constant bigger than 1. By recording the ratio $v=\frac{a+\sum_{m=1}^M\sigma_m}{\mu\lambda}$, we have $a=v\mu\lambda-\sum_{m=1}^M\sigma_m\ge (v-1)\mu\lambda$. In the following, we discuss two situations based on the different possible values of $v$.

\subsection{$v\le1$}\label{appsec1.1}
As $\mu\lambda=\mu(a-\sum_{i\in I_u}h_i)\le\mu a$, so $\frac{1}{a+\sum_{m=1}^M \sigma_m}=\frac{1}{v\mu\lambda}\ge\frac{1}{v\mu a}$, $r=v\mu$. Here, $v$ takes the range from $\frac{1}{\mu}$ to 1. If we set $\mu=2$, for $\epsilon=10^{-6}$ and small $\frac{a}{b}=10^{-8}$, the biggest $t\ge\frac{ln\frac{a\epsilon}{b}}{ln\frac{v\mu-1}{v\mu}}\ge\frac{ln10^{-14}}{ln\frac{1}{2}}\approx46.5$. Thus at most 47 iterations are needed to ensure the convergence of the algorithm. Moreover, if $v$ is smaller, the $t$ will decrease, meaning that $a+\sum_{m=1}^M\sigma_m$ is close to $a$, thus the step-size $\eta'$ is similar to original $\eta$. The relation between different $v$ and iteration number is shown in Fig.~\ref{fig:Complexity1}. In implementation we use $\lceil\frac{\ln\frac{a\epsilon}{b}}{\ln\frac{v\mu-1}{v\mu}}\rceil$ as the iteration number.

\subsection{$v>1$}\label{appsec1.2}
If we choose $r=v\mu$, then $\ln\frac{r-1}{r}$ will be close to 0 if $v$ grows. Thus the iteration number will increase rapidly with $v$. But we have already explained that $a=v\mu\lambda-\sum_{m=1}^M\sigma_m\ge(v-1)\mu\lambda$, thus $\mu\lambda\le\frac{a}{v-1}$. So $a+\sum_{m=1}^M\sigma_m\le a+\mu\lambda\le \frac{v}{v-1}a$, thus $\frac{1}{a+\sum_{m=1}^M\sigma_m}=\frac{1}{ra}\ge\frac{v-1}{va}$, $r=\frac{v}{v-1}$. This $r$ is more reasonable when $v$ is big. By substituting $r=\frac{v}{v-1}$ into Eq.~\eqref{equation17}, we know that the actual bound of iteration number should be $t\ge\frac{\ln\frac{a\epsilon}{b}}{\ln\frac{1}{v}}$. We provide the function graph as Fig.~\ref{fig:Complexity2} to observe the relation between $v$ and $t$. As shown in the line with cross, bigger $v$ suppresses iteration number quickly. It means if the original value $a$ is big enough, the perturbation won't affect the descent step-size too much. $r$ will be close to 1, driving $\ln\frac{r-1}{r}$ to negative infinity easily. Then the worst case iteration complexity $t$ will drop quickly. The algorithm just takes a little more steps to reach the convergence. It's in accordance with our intuition. 

\begin{figure*}[h!]\label{fig:Complexity}
\centering
\subfigure[Number of iterations when $v\le1$]{
\label{fig:Complexity1}
\includegraphics[width=6cm]{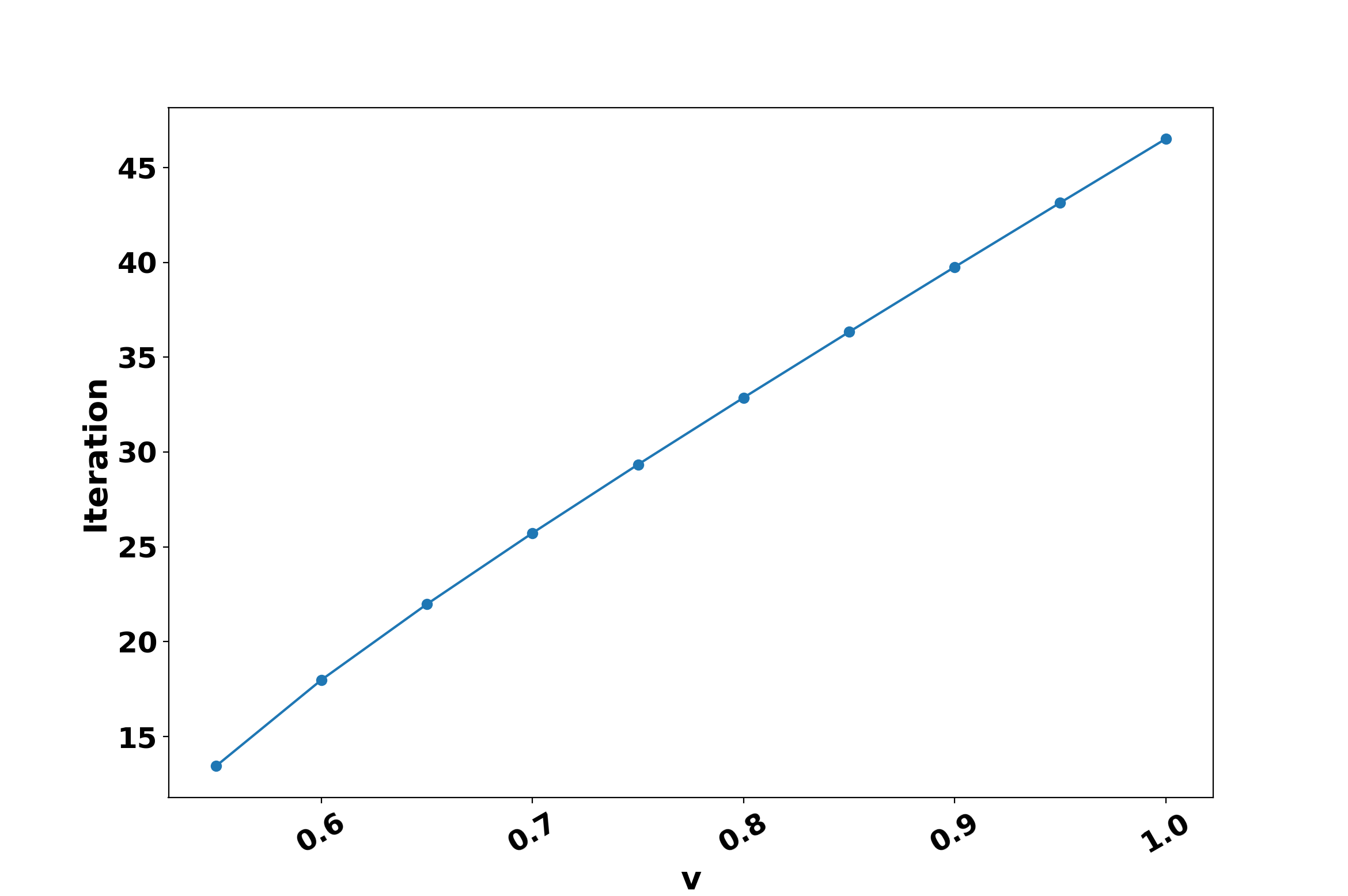}
}
\subfigure[Number of iterations when $v>1$]{
\label{fig:Complexity2}
\includegraphics[width=6cm]{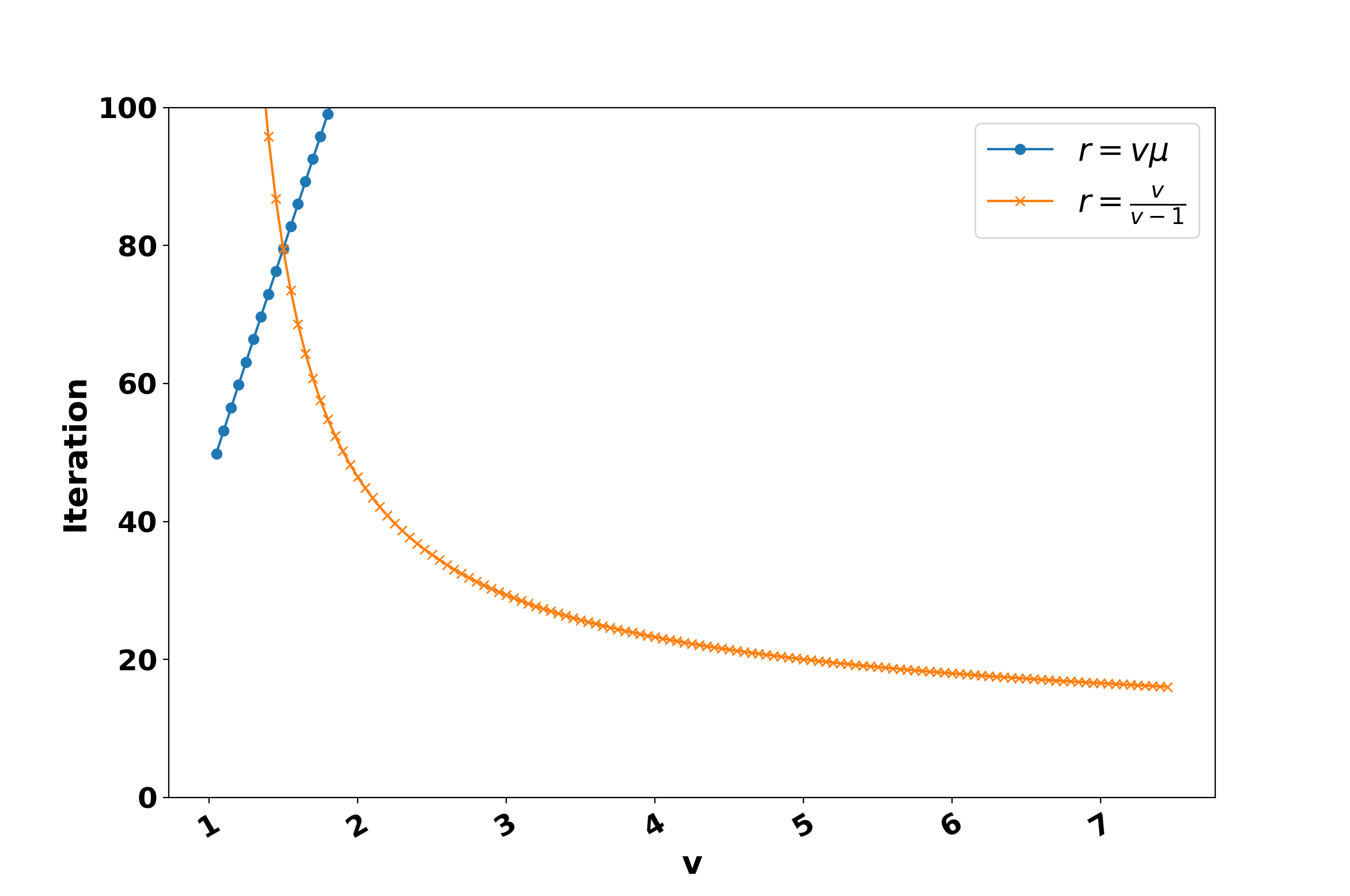}
}
\caption{Number of iterations with respect to different $v$.}
\end{figure*}

However, when $v$ is close to 1, the number of iterations will be very large. Taking the conclusion from Appendix~\ref{appsec1.1}, we also plot the relation between $v$ and the number of iterations counted by $r=v\mu$ in Fig.~\ref{fig:Complexity2} with the dotted line. It is shown that choosing $r=v\mu$ is properer when $v$ is close to 1. By calculating $\frac{v}{v-1}=v\mu$, we get the intersection point $v=\frac{\mu+1}{\mu}$. Finally the algorithm just takes $t_{\min}=\min(\lceil\frac{\ln\frac{a\epsilon}{b}}{\ln\frac{v\mu-1}{v\mu}}\rceil, \lceil\frac{\ln\frac{a\epsilon}{b}}{\ln\frac{1}{v}}\rceil)$ steps to reach convergence. Notably, when $\mu=2$, $r=\frac{\mu+1}{\mu}*\mu=3$,  $t\ge\frac{\ln\frac{a\epsilon}{b}}{\ln\frac{r-1}{r}}=\frac{\ln10^{-14}}{\ln\frac{2}{3}}\approx79.5$, or 80 iterations will guarantee convergence of the algorithm.

All the above analyses show that the iteration complexity is acceptable even under extreme conditions. In practice we rarely reach the worst case in classification tasks. We can adjust the step size by judging $v$. 

In addition, our conclusion is feasible for regression tasks. We follow the premise that $w^*\le0$ and define $y'\triangleq \max_i y_i-\hat{y}_i$. For regression tasks, MSE loss has constant second-order gradient $h_i=1$. Thus we have
\begin{equation}\label{equation19}
\begin{aligned}
w^*=-\frac{\sum_{i\in I_u} g_i}{\sum_{i\in I_u} h_i + \lambda}
\ge-\frac{\sum_{i\in I_u}(y_i-\hat{y}_i)}{\vert I_u\vert}
\ge-\frac{\vert I_u\vert*(y')}{\vert I_u\vert}
=-y'.
\end{aligned}
\end{equation}
Retaking conclusions from Section~\ref{appsec1.1} and Section~\ref{appsec1.2}, at most 80 iterations are needed to converge if the dataset has an extreme large $y'$ such as $1\times10^8$. However, this great difference between initial prediction $\hat{y}_i$ and true label is rare in real-world applications. Even for the case where one really wants to predict large values, then a simple shift of data values or scaling can drift target values to a proper intersection. Then our gradient descent method can be performed the same as above. Notice that all the extreme values are just possible at the beginning. Since XGBoost learns residual error after the first tree, leaf weights on latter trees will be smaller in orders of magnitude, thus fewer iterations would be sufficient enough to guarantee the convergence. The above conclusion can also be proved when $w^*\ge0$ with a similar argument. Therefore, our distributed optimization method with positive perturbation on step size is feasible and provable lossless.

\section{Approximation of Division}\label{appsec2}
Goldschmidt's approach \cite{goldschmidt1964applications} uses multiplication and addition (or subtraction) to approximate \emph{division}, which is mentioned in \cite{fang2020hybrid}. Unfortunately, no detailed implementation is provided. Instead, we implemented a similar approach, Newton's method. They both require iterations and combinations of many computation primitives.

Suppose the problem is to compute $\frac{c}{d}$. To avoid direct \emph{division}, the algorithm needs to get $\frac{1}{d}$ somehow and does multiplication with $c$. In Newton's method, it sets a function $f(x)=d-\frac{1}{x}$. By using taylor-expansion it becomes $f(x)=f(x^{(0)})+f'(x^{(0)})(x-x^{(0)})$. If we set $f(x)=0$, then $x=x^{(1)}=x^{(0)}-\frac{f(x^{(0)})}{f'(x^{(0)})}=x^{(0)}(2-dx^{(0)})$ is obtained. This expression can be iterated arbitrarily to get $x^{(n)}$ as the following
\begin{equation}\label{equation20}
x^{(n)}=x^{(n-1)}(2-dx^{(n-1)}).
\end{equation}

Proper initialization of $x^{(0)}$ will guarantee the convergence of $x^{(n)}$ to $\frac{1}{d}$. Because XGBoost requires convex loss function $l(\cdot)$, it has $\sum_{i\in I_u}h_i+\lambda>0$. This is the denominator $d$ during seeking $\frac{c}{d}$ in computation of loss reductions of XGBoost, thus the converged value $\frac{1}{d}>0$. From Eq.~\eqref{equation20} we know that if $x^{(n-1)}<0$, $x^{(n)}$ will be smaller than 0, which contradicts to $\frac{1}{d}>0$. Thus it has $x^{(n)}>0$ for arbitrary $n$. Besides, from $x^{(n)}=x^{(n-1)}(2-dx^{(n-1)})>0$, we have $0<dx^{(n-1)}<2$. Hence we find that $0<dx^{(0)}<2$ holds for $n=1,2,...$. For security concerns, the original value of $d$ cannot be restored to determine the range of $x^{(0)}$. Participants should cooperate to generate the initial point $x^{(0)}$ which is as close as $\frac{2}{d}$. 

Indeed, the local data $\langle d\rangle^m$ can be expressed as $d'\times10^{\mu_m},d'\in(0.1,1)$. Here, we figure out one method by counting the order of magnitude. That is, each participant $P_m$ computes the order of magnitude $\mu_m$ of $\langle d\rangle^m$ locally. After that they send their $\mu_m$s to $P_1$. $P_1$ will record the biggest $\mu_m$ as $\mu_{\max}$, then it sets $x^{(0)}=10^{-(\mu_{\max}+1)}$. $dx^{(0)}$ will range between 0 and 2 when the number of total participants is less than 20, which is common in current vertically federated settings. Notice that a extremely small $x^{(0)}$ can guarantee $0<dx^{(0)}<2$, but will slow down the convergence to a steady $x^{(n)}$. So we need to estimate the original order of $d$, thus $x^{(0)}$ will be closer to $\frac{1}{d}$. Then every $P_m$ gets its $\langle x^{(0)}\rangle^m=\frac{x^{(0)}}{M}$ from $P_1$. In SS, each iteration of Eq.~\eqref{equation16} in $P_m$ can be written as:
\begin{equation}\label{equation21}
    \langle x^{(n)}\rangle^m=\langle x^{(n-1)}\rangle^m\otimes(\frac{2}{M}-\langle d\rangle^m\otimes\langle x^{(n-1)}\rangle^m).
\end{equation}
Here $\frac{2}{M}$ means equally distributing the constant 2 to $M$ participants. Thus all terms on the same position in Eq.~\eqref{equation21} can reconstruct the term in Eq.~\eqref{equation20}. Also, one \textbf{MUL} between final $x^{(n)}=\frac{1}{d}$ and $c$ is needed. We represent the whole bunch of iterations through convergence as \textbf{DIV}. If $n$ iterations of Eq.~\eqref{equation20} are needed to ensure the convergence, then one \textbf{DIV} takes $n+n+1=2n+1$ \textbf{MUL}s.
\end{document}